\documentclass[journal]{IEEEtran}
\usepackage[utf8]{inputenc}

\usepackage{array}
\usepackage{amsmath,amssymb,amsfonts}
\usepackage{algorithmic}
\usepackage{float}
\usepackage{placeins}
\usepackage{cite}
\usepackage{graphicx}
\usepackage{xcolor}
\usepackage{booktabs}
\usepackage{multirow}
\usepackage{balance}
\usepackage[hidelinks]{hyperref}
\graphicspath{{./}}

\AtBeginDocument{%
  \setlength{\abovedisplayskip}{4.5pt plus 1.5pt minus 1.5pt}%
  \setlength{\belowdisplayskip}{4.5pt plus 1.5pt minus 1.5pt}%
  \setlength{\abovedisplayshortskip}{2.5pt plus 1pt minus 1pt}%
  \setlength{\belowdisplayshortskip}{3.5pt plus 1pt minus 1pt}%
}

\newenvironment{spadecompact}{%
  \setlength{\abovedisplayskip}{3pt plus 1pt minus 1pt}%
  \setlength{\belowdisplayskip}{3pt plus 1pt minus 1pt}%
  \setlength{\abovedisplayshortskip}{2pt plus 0.5pt minus 0.5pt}%
  \setlength{\belowdisplayshortskip}{2pt plus 0.5pt minus 0.5pt}%
  \setlength{\jot}{1.5pt}%
  \allowdisplaybreaks[1]%
}{}

\floatstyle{ruled}
\newfloat{algorithm}{tbp}{loa}
\floatname{algorithm}{Algorithm}

\newtheorem{assumption}{Assumption}
\newtheorem{lemma}{Lemma}
\newtheorem{theorem}{Theorem}
\newtheorem{corollary}{Corollary}

\newtheorem{remark}{Remark}

\makeatletter
\newcommand{\spade@theorembodyfont}{\rmfamily}
\let\spade@orig@begintheorem\@begintheorem
\let\spade@orig@opargbegintheorem\@opargbegintheorem
\renewcommand{\@begintheorem}[2]{%
  \spade@orig@begintheorem{#1}{#2}\spade@theorembodyfont}
\renewcommand{\@opargbegintheorem}[3]{%
  \spade@orig@opargbegintheorem{#1}{#2}{#3}\spade@theorembodyfont}
\let\spade@assumption\assumption
\let\spade@lemma\lemma
\let\spade@theorem\theorem
\let\spade@corollary\corollary
\renewcommand{\assumption}{%
  \renewcommand{\spade@theorembodyfont}{\itshape}\spade@assumption}
\renewcommand{\lemma}{%
  \renewcommand{\spade@theorembodyfont}{\itshape}\spade@lemma}
\renewcommand{\theorem}{%
  \renewcommand{\spade@theorembodyfont}{\itshape}\spade@theorem}
\renewcommand{\corollary}{%
  \renewcommand{\spade@theorembodyfont}{\itshape}\spade@corollary}
\makeatother

\newcommand{\E}{\mathbb{E}}
\newcommand{\R}{\mathbb{R}}
\newcommand{\one}{\mathbf{1}}
\newcommand{\range}{\operatorname{range}}

\newcommand{\datasetnote}[1]{\footnote{\raggedright\fontsize{7}{8}\selectfont\urlstyle{rm}\color{blue}\nolinkurl{#1}}}
\newcommand{\diag}{\operatorname{diag}}

\begin{document}

\title{SPADE-DFL: Communication-Efficient Decentralized Federated Learning via Derivative-Free Linearized ADMM} 
\author{Mengli Wei, Mengkai Zhu, Jiawen Chen, Wenwu Yu and Duxin Chen %
\thanks{Mengli Wei, Mengkai Zhu, Jiawen Chen and Duxin Chen are with the School of  Mathematics, Southeast University, Nanjing 210096, China (e-mail: weimengli@seu.edu.cn).}
\thanks{Wenwu Yu is with the School of Mathematics, Jiangsu Province Scientific Research Center for Applied Mathematics, Southeast University, Nanjing 210096, China. (e-mail: wwyu@seu.edu.cn).}
}
\maketitle

\begin{abstract}
Reducing communication in derivative-free decentralized learning requires controlling the disagreement accumulated over multiple local updates. This paper develops SPADE-DFL, a primal--dual method that allows the number of local function-value updates between neighbor exchanges to grow with the computation budget while preserving the nonprivate convergence order. For smooth nonconvex objectives under uniform query-moment bounds, the prescribed nonprivate schedule achieves a time-averaged stationarity and consensus bound of $\mathcal{O}(T^{-1/3})$ using only $\Theta(T^{2/3})$ communication rounds, where $T$ is the number of local updates per client. For private training, the accumulated data-dependent increment is isolated from the graph correction, allowing one protected state per client and round to generate all outgoing messages. We prove client-level differential privacy for the full interactive transcript and quantify the resulting optimization error over a finite horizon. Experiments on four classification tasks show that SPADE-DFL achieves higher mean test accuracy than existing decentralized learning methods.
\end{abstract}


\section{Introduction}
Decentralized federated learning (DFL) trains a shared model through neighbor exchanges while keeping data local~\cite{zehtabi2025decentralized}. Communication constraints motivate clients to perform several local updates before exchanging model information~\cite{wu2025localupdates, alghunaim2024led}.
The central question is how much communication local computation can replace without compromising progress toward the shared objective~\cite{ren2026communication}.
	
When gradients are unavailable, local computation relies on randomized function evaluations~\cite{ghadimi2013stochastic,ye2025hessian,huang2024zoadmm}.
Single-point methods extend this information model to decentralized learning~\cite{mhanna2023single, mhanna2024zero}. 
Local models can nevertheless drift apart during training, so additional queries offer progress at the cost of a growing coordination problem. 
Compression can reduce the information transmitted in an exchange,
but the number of exchanges remains part of the communication cost~\cite{song2022compressedgt,nassif2025def,he2023unbiased}. 
The question is therefore more specific.
\emph{How much communication can local loss evaluations replace while preserving the convergence order of collective learning?}
	
Local computation reduces the frequency of neighbor exchanges but also permits local models to drift apart before the next exchange. With single-point function-value information, the smoothing radius affects the approximation error of the local direction. The local training length, stepsize, smoothing radius and network gain therefore need to be selected jointly. Their interaction determines whether reducing the exchange frequency also reduces the total communication required to reach a given optimization accuracy.
	
Privacy adds to this dependence on communicated states. Reconstruction attacks show that model information can reveal local training data~\cite{guo2026gradient}. 
Distributed privacy mechanisms exploit graph structure~\cite{rizk2023privacy,allouah2024decor}, including
formulations based on noisy ADMM iterations~\cite{cyffers2023privateadmm}.
Each released state summarizes a local query sequence whose influence continues through subsequent neighbor updates. 
Its protective perturbation therefore alters the feedback driving later computation.
Privacy analysis follows this influence through the complete interaction using composition of successive releases~\cite{bun2016concentrated,mironov2017renyi}.
	
The construction separates the accumulated local learning increment from the correction determined by preceding exchanges. Decentralized formulations of the alternating direction method of multipliers (ADMM) express agreement through constraints on neighboring models~\cite{boyd2011distributed,shi2014linearized}. Linearization makes the local primal update explicit~\cite{ling2015dlm}, while curvature aided formulations permit local primal and dual steps without inner communication loops~\cite{li2023curvature}. Local training further extends the computation performed between exchanges~\cite{ren2026communication}. Reference estimates also provide a way to reuse information in zeroth-order optimization~\cite{gautam2024mezosvrg}. 
Building on these ideas, we construct local directions from a component memory that preserves the conditional mean of fresh single-point estimates. The ADMM correction is held fixed throughout each local stage. An exact message recursion relates accumulated local updates to subsequent network disagreement and supports the joint selection of the local training length and round gain.
The placement of privacy protection also matters since clipping can bias stochastic updates~\cite{koloskova2023clipping}. We therefore isolate the accumulated data dependent increment from the correction fixed by the preceding transcript. Clipping this increment before adding calibrated Gaussian noise produces one protected state per client and round from which all outgoing messages are constructed. Retaining the explicit graph correction makes the effect of each release visible in subsequent exchanges. The message recursion determines how the local stage can grow while preserving the convergence order under the nonprivate parameter schedule. For private training, a finite horizon bound accounts for clipping and Gaussian releases in the stationarity and consensus criterion.
The main contributions are as follows.
\begin{enumerate}
	\item \textbf{Communication savings from local loss evaluations.}
	We establish communication savings from local loss evaluations for smooth nonconvex objectives under uniform query moment control. The prescribed nonprivate schedule yields a time averaged stationarity and consensus bound of $\mathcal{O}(T^{-1/3}+\tau^2/T)$ after $T=K\tau$ local updates per client. The drift term permits $\tau=\Theta(T^{1/3})$ local updates per exchange while retaining the $\mathcal{O}(T^{-1/3})$ order. For a joint tolerance $\varepsilon_{\mathrm{stat}}$, this reduces the sufficient communication bound from $\mathcal{O}(\varepsilon_{\mathrm{stat}}^{-3})$ for the same method with $\tau=1$ to $\mathcal{O}(\varepsilon_{\mathrm{stat}}^{-2})$ rounds. 
		
	\item \textbf{Single-point local training and network dynamics.}
SPADE-DFL incorporates a component-memory single-point estimator into local training ADMM. Under uniform query-moment bounds, the memory correction preserves the conditional mean of fresh single-point estimates and admits a controlled second moment. An exact message recursion yields a modal representation of network disagreement. In the synchronized setting, the homogeneous disagreement modes are Schur stable exactly when $0<\tau\beta\eta\mu\rho\lambda_N<8/3$, where $\lambda_N$ is the largest graph Laplacian eigenvalue. This characterization determines the admissible round gain used in the local training schedule.
		
	\item \textbf{Protecting local computation through the state used for coordination.}
	We isolate the accumulated private update from the graph correction fixed by the observed history. One Gaussian release per client and round protects the clipped update, while all outgoing messages follow by deterministic processing. We prove differential privacy for the full interactive transcript under replacement of an entire client dataset, accounting for effects propagated through subsequent neighbor responses. A finite horizon analysis traces release perturbations through the feedback governing objective descent. The resulting bound quantifies the optimization cost of privacy under local drift, linking information disclosed during communication to collective learning accuracy.
\end{enumerate}
	
Sections~II through IV cover related work, the method and its analysis. Section~V presents the numerical study before Section~VI concludes the paper.

\section{Related Work}
	
\subsection{Communication in Decentralized Learning}
The benefit of local computation depends on whether progress made between exchanges survives the disagreement generated by heterogeneous objectives. Analyses of local gradient descent and gradient tracking make this dependence explicit~\cite{wu2025localupdates}. ProxSkip establishes communication acceleration through randomized synchronization for strongly convex problems~\cite{mishchenko2022proxskip}. Local exact diffusion incorporates bias correction into local training~\cite{alghunaim2024led}, while local training ADMM holds a neighbor correction fixed over several stochastic updates~\cite{ren2026communication}. 
A differentially private variant applies clipping and Gaussian
perturbation to stochastic gradients at each local step while
retaining one neighbor-exchange phase per round~\cite{ren2026ltadmmdp}.
Its analysis establishes privacy under single-record addition
or removal and a stationarity bound for nonconvex objectives.
The communication cost also depends on the mixing protocol. DSGD with CECA uses an exact consensus schedule~\cite{ding2023dsgdceca}, whereas tree based push pull limits the number of active neighbors~\cite{you2024btpp}. Related analyses show how heterogeneity correction improves the dependence of transient behavior on topology~\cite{yuan2023heterogeneity}. Work on DFL examines this relationship through aggregation weight optimization~\cite{zhai2026aggregation} and the stability of decentralized training~\cite{sun2026generalization}. Compression addresses the amount transmitted at each exchange, with its overall benefit depending on the convergence cost of smaller messages~\cite{he2023unbiased}. Compressed gradient tracking treats communication over directed networks~\cite{song2022compressedgt}. Differential error feedback reuses compression residuals~\cite{nassif2025def}, while BEER develops compression with gradient tracking for nonconvex objectives~\cite{zhao2022beer}. MoTEF uses momentum tracking to control stochastic error within an error feedback scheme~\cite{islamov2025motef}. For objectives that vary over time, compressed distributed methods connect communication to online regret~\cite{li2024onlinecompressed}.
	
\subsection{Learning from Function Values}
Zeroth-order optimization studies how function evaluations can provide enough directional information for learning~\cite{ghadimi2013stochastic}. Single-point distributed methods use one noisy evaluation per update~\cite{mhanna2023single,mhanna2024zero}. The accuracy of the inferred direction can also be improved through curvature information~\cite{ye2025hessian} or stochastic ADMM constructions with explicit query complexity~\cite{huang2024zoadmm}. In networked problems, compressed stochastic methods incorporate transmission error into the analysis~\cite{hua2026distributed}. Quantized gradient tracking with deterministic zeroth-order estimates yields linear convergence under the Polyak \L{}ojasiewicz condition~\cite{xu2024quantized}. Function evaluations also change the implementation cost of local learning. FedZO performs several local updates between server aggregations~\cite{fang2022communication}, while DeComFL represents communicated information by scalars to remove the dependence of the payload on model dimension~\cite{li2025decomfl}. For language model adaptation, MeZO avoids backpropagation by estimating directions from paired forward evaluations~\cite{malladi2023mezo}. Its SVRG extension uses reference information to improve the optimization process~\cite{gautam2024mezosvrg}. These developments connect the choice of estimator to the work performed by the local solver. For decentralized solvers, ADMM expresses coordination through equality constraints on neighboring models~\cite{boyd2011distributed,shi2014linearized}. Linearization replaces the local primal solve with an explicit approximation~\cite{ling2015dlm}. Local training allows this computation to continue between exchanges~\cite{ren2026communication}. Curvature aided methods use gradient directions or Newton approximations within local primal and dual updates without inner communication loops for convex composite objectives~\cite{li2023curvature}.

\begin{table*}[!t]
    \centering
    \caption{Selected algorithmic features of the methods discussed in Related Work.}
    \label{tab:related_comparison}
    \begingroup
    \footnotesize
    \setlength{\tabcolsep}{2.5pt}
    \renewcommand{\arraystretch}{1.12}
    \begin{tabular*}{\textwidth}{@{\extracolsep{\fill}}lcccccc@{}}
        \toprule
        \multirow{2}{*}{Method}
        & Communication & Local & Single-point & Component & ADMM & Differential \\
        & reduction & updates & oracle & memory & updates & privacy \\
        \midrule

        Local training~\cite{wu2025localupdates,alghunaim2024led,mishchenko2022proxskip,fang2022communication,li2025decomfl}
        & $\checkmark$ & $\checkmark$ & & & & \\

        LT-ADMM~\cite{ren2026communication}
        & $\checkmark$ & $\checkmark$ & & & $\checkmark$ & \\

        LT-ADMM-VR~\cite{ren2026communication}
        & $\checkmark$ & $\checkmark$ & & $\checkmark$ & $\checkmark$ & \\

        LT-ADMM-DP~\cite{ren2026ltadmmdp}
        & $\checkmark$ & $\checkmark$ & & & $\checkmark$ & $\checkmark$ \\

        Communication-saving methods~\cite{ding2023dsgdceca,you2024btpp,song2022compressedgt,nassif2025def,zhao2022beer,islamov2025motef,li2024onlinecompressed,hua2026distributed,xu2024quantized}
        & $\checkmark$ & & & & & \\

        Nonprivate optimization~\cite{yuan2023heterogeneity,zhai2026aggregation,ye2025hessian,malladi2023mezo,gautam2024mezosvrg}
        & & & & & & \\

        Single-point methods~\cite{mhanna2023single,mhanna2024zero}
        & & & $\checkmark$ & & & \\

        ADMM methods~\cite{huang2024zoadmm,shi2014linearized,ling2015dlm}
        & & & & & $\checkmark$ & \\

        ADMM variants~\cite{li2023curvature,li2024adqsp}
        & $\checkmark$ & & & & $\checkmark$ & \\

        Private learning~\cite{rizk2023privacy,allouah2024decor,wang2026ping,cyffers2022muffliato,zhang2024dpzero,gong2025pazo}
        & & & & & & $\checkmark$ \\

        Private ADMM~\cite{cyffers2023privateadmm}
        & & & & & $\checkmark$ & $\checkmark$ \\

        DP-FedSAM (top $k$)~\cite{shi2025dpfedsam}
        & $\checkmark$ & $\checkmark$ & & & & $\checkmark$ \\

        \midrule
        \textbf{SPADE-DFL}
        & $\checkmark$ & $\checkmark$ & $\checkmark$
        & $\checkmark$ & $\checkmark$ & $\checkmark$ \\
        \bottomrule
    \end{tabular*}
    \endgroup
\end{table*} 
    
\subsection{Privacy across Distributed Interactions}
Privacy in decentralized learning depends on what an observer can infer from a history of exchanges. Graph homomorphic noise connects the protection of communicated states to learning performance~\cite{rizk2023privacy}. Adaptive differentially quantized subspace perturbation exploits the consensus structure to address privacy with compressed communication~\cite{li2024adqsp}. Positive incentive noise designs use graph interactions to regulate the resulting utility~\cite{wang2026ping}. Muffliato analyzes privacy amplification under a pairwise network observation model~\cite{cyffers2022muffliato}. DECOR instead constructs Gaussian perturbations that cancel across neighboring clients using shared secret randomness~\cite{allouah2024decor}. Private ADMM connects iterative optimization to a noisy fixed point formulation~\cite{cyffers2023privateadmm}. Concentrated privacy and R\'enyi privacy provide composition rules for successive releases~\cite{bun2016concentrated,mironov2017renyi}. The effect of a release also depends on how the local update is prepared. Clipping can introduce persistent stochastic bias~\cite{koloskova2023clipping}. DP-FedSAM examines private local training through sharpness aware optimization with update sparsification~\cite{shi2025dpfedsam}. For learning from function values, DPZero develops private model adaptation without backpropagation~\cite{zhang2024dpzero}. PAZO uses public data to guide the private gradient approximation under a similarity assumption between the two data sources~\cite{gong2025pazo}.

SPADE-DFL quantifies the communication savings attainable from local single-point function evaluations. Under uniform query-moment bounds, the prescribed nonprivate schedule permits
the local training length to grow with the computation budget while preserving the stationarity and consensus convergence order. For private training, the accumulated local update is clipped and perturbed once per round, and client-level privacy is established for the full interactive transcript under replacement of an entire fixed-size local dataset.
Table~\ref{tab:related_comparison} summarizes these distinctions from the methods discussed above.

\section{Problem Formulation and Methodology}
\subsection{Problem Formulation}
	Consider a decentralized empirical-risk minimization system in which training samples remain distributed across $N\ge2$ clients and no central server has direct access to their union \cite{mcmahan2017communication,kairouz2021advances}.
	The $N$ clients are connected by a graph $\mathcal{G}=(\mathcal{V},\mathcal{E})$, where $\mathcal{V}=\{1,\ldots,N\}$ and $\mathcal{E}$ is the edge set.
	Client $i$ communicates only with neighbors in $\mathcal{N}_i:=\{j:\{i,j\}\in\mathcal{E}\}$ and stores $\mathcal D_i=\{\xi_{i,h}\}_{h=1}^{m_i}$, where the sample $\xi_{i,h}$ induces the component loss $f_{i,h}:\R^d\to\R$.
	The client-specific empirical risk and the client-uniform learning objective are defined respectively as
	\begingroup
	\allowdisplaybreaks
	\begin{align}
	f_i(x)&=\frac{1}{m_i}\sum_{h=1}^{m_i}f_{i,h}(x),&
	F(x)&=\frac{1}{N}\sum_{i=1}^{N}f_i(x),
	\label{eq:objectives}
	\end{align}
	\endgroup
	which separates the within-client empirical average from the network-level aggregation. 
	Data heterogeneity and the local sample size are absorbed into $f_i$. 
	The component losses $f_{i,h}$, local empirical risks $f_i$ and aggregate objective $F$ may be nonconvex.
	At a queried model point, client $i$ observes component loss values but has no access to $\nabla f_{i,h}$ or $\nabla f_i$.
	This information pattern arises when the local objective is exposed through the external black-box interface or its derivatives are computationally prohibitive~\cite{mhanna2023single, huang2024zoadmm, ye2025hessian}. 
	
	To implement the common-model objective through neighbor exchanges, assign a local replica $x_i$ to each client. The resulting consensus formulation can be obtained from~\cite{boyd2011distributed} as
	\begingroup
	\allowdisplaybreaks
	\begin{align}
	\min_{\{x_i\}_{i=1}^{N}}\quad
	\frac{1}{N}\sum_{i=1}^{N}f_i(x_i)
	\quad\text{s.t.}\quad
	x_i=x_j,\quad \{i,j\}\in\mathcal E .
	\label{eq:consensus-problem}
	\end{align}
	\endgroup
	Since \eqref{eq:consensus-problem} separates local objectives and couples their replicas only along graph edges, connectivity makes the edge constraints equivalent to $x_1=\cdots=x_N$, thereby recovering $\min_x F(x)$ and enabling decentralized ADMM through local loss evaluations and neighbor messages. The required objective regularity and network conditions are stated in Assumptions~\ref{ass:objective} and~\ref{ass:network}, respectively.
	
	\begin{assumption}[\hspace{-0.02cm}\cite{ghadimi2013stochastic,mhanna2023single}]
		\label{ass:objective}
		Each component loss is continuously differentiable and $L_f$-smooth, i.e., $\|\nabla f_{i,h}(x)-\nabla f_{i,h}(y)\|\le L_f\|x-y\|$, $\forall x,y\in\R^d$. Furthermore, $F$ is bounded below by $F_\star>-\infty$.
	\end{assumption}
	
	\begin{assumption}[\hspace{-0.02cm}\cite{shi2014linearized}]\label{ass:network}
		The communication graph $\mathcal G$ is fixed, undirected, unweighted, and connected. Messages are exchanged synchronously over its edges. 
	\end{assumption}
	
	For the graph representation, orient each edge arbitrarily and let $B\in\R^{N\times|\mathcal E|}$ be the resulting incidence matrix. 
	For an edge $e=(i,j)$, set $B_{i,e}=1$, $B_{j,e}=-1$, and all other entries in column $e$ to zero.
	Then $L_{\mathcal G}=BB^\top$ and $0=\lambda_1<\lambda_2\le\cdots\le\lambda_N$ define the graph Laplacian and its ordered eigenvalues, respectively. 
	Suppressing unambiguous Kronecker products with $I_d$, denote $\Pi=\left(I_N-\frac{1}{N}\one\one^\top\right)\otimes I_d$.
	Before specifying the primal--dual recursion, we characterize the single-point information available at a local model and establish the properties of a memory-based direction formed from such queries.
	
	\begingroup
	\linespread{0.94}\selectfont
	\setlength{\abovedisplayskip}{1pt plus 0.5pt minus 0.5pt}
	\setlength{\belowdisplayskip}{1pt plus 0.5pt minus 0.5pt}
	\setlength{\abovedisplayshortskip}{0.5pt plus 0.5pt}
	\setlength{\belowdisplayshortskip}{0.5pt plus 0.5pt}
	
	\subsection{Single-Point Estimation and Memory Properties}
	\label{subsec:single-point-estimation}
	Each sampled component provides one noisy function value from which a derivative-free direction can be formed. 
	Let $\mathcal F$ denote the information available before a query at an $\mathcal F$-measurable point $x$. Client $i$ observes $\widetilde f_{i,h}(x+\mu u)=f_{i,h}(x+\mu u)+\zeta$, where $\mu>0$ is the smoothing radius, $u\in\R^d$ is the current query direction and $\zeta$ is the query noise, respectively.
	The following assumption adapts the one-point perturbation model to the generated query points.
	
	\begin{assumption}[\hspace{-0.02cm}\cite{mhanna2023single, mhanna2024zero}]\label{ass:oracle}
		For each generated query at an $\mathcal F$-measurable point, there exist positive constants $c_u,R_u,M_2>0$ such that $\E\!\left[u\mid\mathcal F\right]=0$, $\E\!\left[uu^\top\mid\mathcal F\right]=c_u I_d$ and $\|u\|\le R_u$ almost surely. 
		Moreover, $\E\!\left[\zeta\mid\mathcal F,u\right]=0$ and $\E\!\left[ |\widetilde f_{i,h}(x+\mu u)|^2\mid\mathcal F \right]\le M_2$.
	\end{assumption}

    The query-moment condition admits smooth objectives on an
unconstrained parameter domain. If $|f_{i,h}(z)|\le B_f$
for every $i,h$ and $z\in\mathbb R^d$, and
$\mathbb E[\zeta^2\mid\mathcal F,u]\le\sigma_\zeta^2$,
conditional centering of the query noise gives $\mathbb E\!\left[
|\widetilde f_{i,h}(x+\mu u)|^2
\mid\mathcal F
\right]
\le B_f^2+\sigma_\zeta^2$.
A concrete example is the squared probability loss
$f(x)=\frac12\|\operatorname{softmax}(Wa)-y\|^2$,
where $x=\operatorname{vec}(W)$, the feature vector $a$
has uniformly bounded norm, and $y$ is a one-hot label.
This loss is globally smooth and takes values in $[0,1]$
throughout the parameter space. Thus, bounded smooth
losses with uniformly bounded conditional noise variance
provide a query-moment bound independent of the query
locations and smoothing radii. Taking $B_f$ and
$\sigma_\zeta$ independent of the training budget also
provides the uniformity required by the convergence rates.
	Under Assumption~\ref{ass:oracle}, the local single-point estimator is
	\begin{align}
	q_{i,h}(x;u,\zeta)
	=\frac{1}{c_u}u\widetilde f_{i,h}(x+\mu u).
	\label{eq:q-generic}
	\end{align}
Each evaluation of (3) requires one function value. The estimator is used without division by $\mu$. At local step $t$ of round $k$, client $i$ evaluates (3) at $x=\phi_{i,k}^{t}$ with smoothing radius $\mu_k$ and direction $u_{i,h,k}^{t}$; the resulting query is denoted by $q_{i,h,k}^{t}$.
	
	In round $k$, client $i$ starts from the released model $x_{i,k}$ and performs $\tau_i$ local updates with mini-batch size $1\le b_i\le m_i$. Set $\phi_{i,k}^0=x_{i,k}$. To reuse component-level single-point information, one memory entry is initialized for every $h\in\{1,\ldots,m_i\}$ by setting $r_{i,h,k}^0=x_{i,k}$, drawing a reference direction $u_{i,h,k}^{\mathrm{ref},0}$, and storing
	\begin{align}
	a_{i,h,k}^0
	&=\frac{1}{c_u}u_{i,h,k}^{\mathrm{ref},0}
	\widetilde f_{i,h}\left(
	r_{i,h,k}^0+\mu_k u_{i,h,k}^{\mathrm{ref},0}
	\right).
	\label{eq:memory-initialization}
	\end{align}
	The initial table average is $\bar a_{i,k}^0:=m_i^{-1}\sum_{h=1}^{m_i}a_{i,h,k}^0$, and this initialization uses one function query per component at the beginning of the round.
	
	Let $\mathcal F_{k,t}$ contain the public transcript, released and auxiliary states, frozen penalties, current local iterates, memory table, and all randomness revealed before local step $(k,t)$. The current mini-batches, query directions, and query noises are excluded. Conditional on $\mathcal F_{k,t}$, each $\mathcal B_{i,k}^t$ is uniform over the subsets of size $b_i$ and is independent of the current query randomization; current batches are independent across clients. With $\bar a_{i,k}^t:=m_i^{-1}\sum_{h=1}^{m_i}a_{i,h,k}^t$, define
	\begin{align}
	v_{i,k}^t
	&=\frac{1}{|\mathcal B_{i,k}^t|}\sum_{h\in\mathcal B_{i,k}^t}
	\left(q_{i,h,k}^t-a_{i,h,k}^t\right)
	+\bar a_{i,k}^t.
	\label{eq:memory-estimator}
	\end{align}
	After forming $v_{i,k}^t$, each sampled pair is replaced by $(r_{i,h,k}^{t+1},a_{i,h,k}^{t+1}):=(\phi_{i,k}^t,q_{i,h,k}^t)$, while unsampled entries remain unchanged and $\bar a_{i,k}^{t+1}$ is recomputed.
	
	To state the oracle and memory errors, define
	\[
	\begin{aligned}
	r_{i,h}(x,\mu)
	&:=
	\mathbb E\!\left[
	q_{i,h}(x;u,\zeta)\mid\mathcal F
	\right]
	-\mu\nabla f_{i,h}(x),\\
	r_{i,k}^t
	&:=
	\mathbb E\!\left[
	v_{i,k}^t\mid\mathcal F_{k,t}
	\right]
	-\mu_k\nabla f_i(\phi_{i,k}^t).
	\end{aligned}
	\]
	For compactness, set $c_r:=\frac{L_fR_u^3}{2c_u}$ and $Q^2:=\frac{R_u^2M_2}{c_u^2}$.
	The two statements below characterize the mean and
	moment properties of the unnormalized single-query estimator under
	Assumption~\ref{ass:oracle} and those of the memory recursion in
	\eqref{eq:memory-initialization}--\eqref{eq:memory-estimator}, respectively.
	
	\begin{lemma}
		\label{lem:oracle_memory}
		Under Assumptions~\ref{ass:objective} and~\ref{ass:oracle}, the following statements hold.
		\begin{enumerate}
			\item[(i)] For each $\mathcal F$-measurable query point $x$, it holds that $\mathbb E\!\left[q_{i,h}(x;u,\zeta)\mid\mathcal F\right]=\mu\nabla f_{i,h}(x)+r_{i,h}(x,\mu)$, where $\|r_{i,h}(x,\mu)\|\le c_r\mu^2$, and $\mathbb E\!\left[\|q_{i,h}(x;u,\zeta)\|^2\right]\le Q^2$.
			
			\item[(ii)] For each local step $(k,t)$, it holds that $\mathbb E\!\left[v_{i,k}^t\mid\mathcal F_{k,t}\right]=\mu_k\nabla f_i(\phi_{i,k}^t)+r_{i,k}^t$, where $\|r_{i,k}^t\|\le c_r\mu_k^2$, while $m_i^{-1}\sum_{h=1}^{m_i}\mathbb E\!\left[\|a_{i,h,k}^t\|^2\right]\le Q^2$ and $\mathbb E\!\left[\|v_{i,k}^t\|^2\right]\le 9Q^2$.
		\end{enumerate}
	\end{lemma}

    Lemma~\ref{lem:oracle_memory}(i) characterizes the conditional mean, smoothing remainder, and second moment of a component query. Lemma~\ref{lem:oracle_memory}(ii) shows that the memory correction preserves the conditional mean of fresh single-point estimates with a controlled second moment. The next subsection incorporates this direction into the local training ADMM recursion.
	
	\subsection{SPADE-DFL}
	\label{subsec:algorithm-design}
	
	\begin{algorithm}[!t]
		\caption{SPADE-DFL at client $i$}
		\label{alg:spade}
        \small
		\begin{algorithmic}[1]
			\REQUIRE Local data $\mathcal D_i$; neighbor set $\mathcal N_i$; public model $x_0$; parameters $\beta,\rho,\tau_i,b_i$; schedules $\{\eta_k,\mu_k,R_k^{\mathrm{clip}},\sigma_{\mathrm{dp},k}\}_{k=0}^{K-1}$.
			\ENSURE Private released model $x_{i,K}$ and auxiliary states $\{z_{ij,K}\}_{j\in\mathcal N_i}$.
			\FOR{$k=0,\ldots,K-1$}
			\STATE Set $\phi_{i,k}^0=x_{i,k}$, freeze $p_{i,k}$ by \eqref{eq:penalty}, and initialize the derivative-free estimator memory by \eqref{eq:memory-initialization}.
			\FOR{$t=0,\ldots,\tau_i-1$}
			\STATE Sample $\mathcal B_{i,k}^t$, evaluate one $q_{i,h,k}^t$ per selected component using \eqref{eq:q-generic}, and form $v_{i,k}^t$ by \eqref{eq:memory-estimator}.
			\STATE Update $\phi_{i,k}^{t+1}$ by \eqref{eq:local-update}, replace the sampled memory pairs, and recompute $\bar a_{i,k}^{t+1}$.
			\ENDFOR
			\STATE Set
			\(y_{i,k+1}=\phi^{\tau_i}_{i,k}\) and
			\(s_{i,k}=-\eta_k\sum_{t=0}^{\tau_i-1}v^t_{i,k}\).
			
			\STATE Clip the accumulated local update, where
			\(\operatorname{clip}_{R}(0)=0\):
			\[
			\begin{aligned}
			\bar{s}_{i,k}
			=\operatorname{clip}_{R_k^{\mathrm{clip}}}(s_{i,k}) =s_{i,k}\min\!\left\{ 1,\frac{R_k^{\mathrm{clip}}}{\lVert s_{i,k}\rVert} \right\}.
			\end{aligned}
			\]
			
			\STATE Independently sample the Gaussian release noise as
			\[
			\nu_{i,k}\sim
			\mathcal{N}\!\left(
			0,\sigma_{\mathrm{dp},k}^{2}I_d
			\right).
			\]
			
			\STATE Compute and release \(x_{i,k+1}\) by
			\eqref{eq:private-release}.
			\STATE Exchange the messages in \eqref{eq:message} and update all $z_{ij,k+1}$ by \eqref{eq:z-update}.
			\ENDFOR
		\end{algorithmic}
	\end{algorithm}
	\endgroup

        \begin{figure*}[!t]
		\centering
		\includegraphics[width=\textwidth]{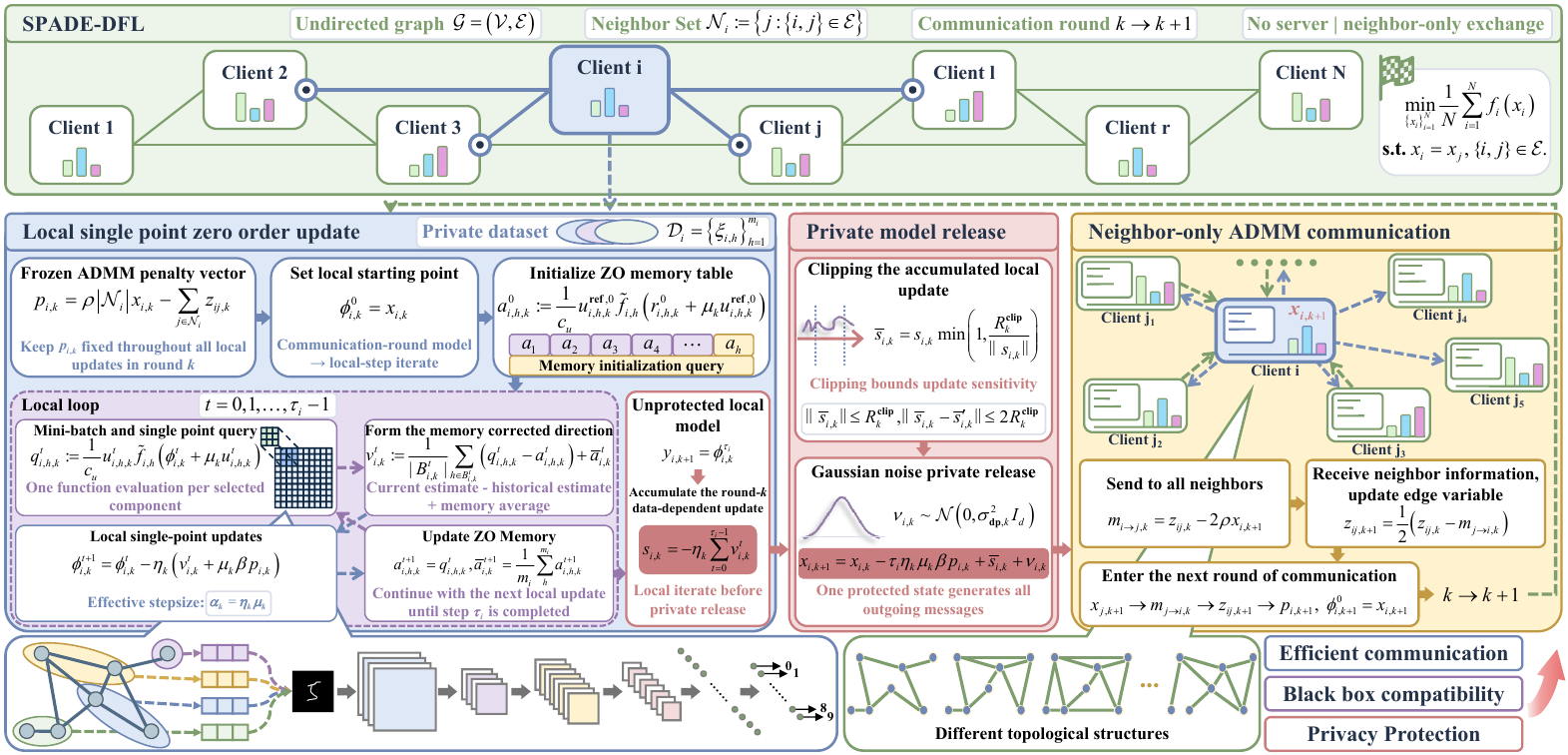}
		\caption{Workflow of SPADE-DFL. In round $k$, each client freezes its ADMM penalty, initializes the one-point estimator memory, and performs $\tau_i$ local derivative-free control-variate updates. The accumulated data-dependent local update is clipped and perturbed once before release, after which all neighbor messages and auxiliary-variable updates are computed from the released state.}
		\label{fig:spade-workflow}
	\end{figure*}
    
	Repeated neighbor synchronization can dominate the cost of derivative-free decentralized learning. To reduce the communication frequency, SPADE-DFL freezes the ADMM correction during a local stage, performs multiple single-query updates, and communicates only after the stage is completed. Consequently, client $i$ executes $\tau_i$ local updates between two consecutive neighbor exchanges, which is illustrated in Algorithm~\ref{alg:spade}.
	Round $k$ uses step size $\eta_k>0$ and smoothing radius $\mu_k>0$, while $\rho,\beta>0$ determine the ADMM correction. 
	All clients use $x_{i,0}=x_0$ and $z_{ij,0}=\rho x_0$ for every $j\in\mathcal N_i$, hence $z_{ij,0}+z_{ji,0}=\rho(x_{i,0}+x_{j,0})$ on each edge.
	The initialization, graph, and parameter schedules are public and data-independent. Fresh local randomization is independent across clients, and release noise is independent of all pre-release variables.

	\subsubsection{Local Derivative-Free Training}
	
	At the beginning of round $k$, client $i$ sets $\phi_{i,k}^0=x_{i,k}$ and freezes the neighbor-dependent ADMM penalty at the current released state as
	\begin{align}
	p_{i,k}
	=\rho|\mathcal N_i|x_{i,k}
	-\sum_{j\in\mathcal N_i}z_{ij,k}.
	\label{eq:penalty}
	\end{align}
	The same $p_{i,k}$ is used throughout all $\tau_i$ local updates and the local model evolves as
	\begin{align}
	\phi_{i,k}^{t+1}
	=\phi_{i,k}^{t}
	-\eta_k\left(v_{i,k}^t+\mu_k\beta p_{i,k}\right),
	\label{eq:local-update}
	\end{align} 
	Let $\alpha_k=\eta_k\mu_k$ and
$\widehat v_{i,k}^{t}=v_{i,k}^{t}/\mu_k$.
The local update equivalently reads $\phi_{i,k}^{t+1}
=
\phi_{i,k}^{t}
-
\alpha_k
\left(
\widehat v_{i,k}^{t}
+
\beta p_{i,k}
\right)$.
Thus, $\alpha_k$ is the effective descent stepsize and
$\alpha_k\beta$ is the coefficient of the frozen ADMM correction
per local step. Under the synchronized parameters settings, $a=\tau\alpha\beta$ is the aggregate network gain
over one communication round. After $\tau_i$ updates, define $y_{i,k+1}:=\phi_{i,k}^{\tau_i}$ and $s_{i,k}:=-\eta_k\sum_{t=0}^{\tau_i-1}v_{i,k}^t$.
	Since $p_{i,k}$ is fixed within the round, it follows that $y_{i,k+1}=x_{i,k}-\tau_i\eta_k\mu_k\beta p_{i,k}+s_{i,k}$. 
	Therefore, the state $y_{i,k+1}$ is intermediate and unreleased.

	\subsubsection{Private Release and Decentralized Communication}
	Before communication, SPADE-DFL sanitizes the accumulated data-dependent update $s_{i,k}$ once, while leaving the explicit ADMM correction unchanged. The release noise is independent of all variables generated before the corresponding release and is conditionally independent across clients and rounds. Specifically, client \(i\) projects the accumulated data-dependent update \(s_{i,k}\) onto the closed Euclidean ball centered at the origin with radius \(R_k^{\mathrm{clip}}\), and denotes the resulting clipped vector by \(\bar{s}_{i,k}\). It then samples a zero-mean Gaussian vector \(\nu_{i,k}\) with covariance \(\sigma_{\mathrm{dp},k}^{2}I_d\). The private released state is
	\begin{align}
	x_{i,k+1} = x_{i,k} -\tau_i\eta_k\mu_k\beta p_{i,k} +\bar{s}_{i,k} +\nu_{i,k}.
	\label{eq:private-release}
	\end{align}
	Define \(e^{\mathrm{cl}}_{i,k}:=\bar{s}_{i,k}-s_{i,k}\), which gives
	the equivalent decomposition
	\(x_{i,k+1}=y_{i,k+1}+e^{\mathrm{cl}}_{i,k}+\nu_{i,k}\).
	The released state is then used to form each outgoing message as
	\begin{align}
	m_{i\to j,k}=z_{ij,k}-2\rho x_{i,k+1}, j\in\mathcal N_i.
	\label{eq:message}
	\end{align}
	After receiving the corresponding message from neighbor $j$, client $i$ updates the directed auxiliary state as
	\begin{align}
	z_{ij,k+1}
	=\frac{1}{2}\left(z_{ij,k}-m_{j\to i,k}\right).
	\label{eq:z-update}
	\end{align}
Client \(i\) completes \(\tau_i\) local updates before each neighbor exchange, with all outgoing messages constructed from the released state \(x_{i,k+1}\). 
Communication is thus amortized over multiple local updates and each directed edge carries one \(d\)-dimensional ADMM message per round, the complete workflow of which is summarized in~Fig.~\ref{fig:spade-workflow}.

\begin{spadecompact}
\section{Theoretical Analysis}
The exact message recursion determines the graph correction from the released history. Conditioning on this history gives the release law used for privacy accounting. The disagreement and descent bounds then yield the communication and query budgets.
		
\subsection{Exact Message Recursion and Transcript Privacy}
Variables without node or edge subscripts denote stacked vectors and Kronecker products with $I_d$ are omitted when unambiguous. 
For each oriented edge $e=(i,j)$, define $\omega_{e,k}:=(z_{ij,k}-z_{ji,k})/2$. Set $\omega_{-1}:=0$ and $\lambda_k:=-B\omega_{k-1}$.
These variables express the neighbor exchanges as a recursion for the graph correction.		
\begin{lemma}\label{lem:message}
Under the common initialization in Algorithm~\ref{alg:spade}, the message rules \eqref{eq:message}--\eqref{eq:z-update} satisfy
\begin{align}
\omega_{k+1}&=\omega_k-\frac{\rho}{2}B^\top x_{k+1}, \label{eq:omega-update}\\
p_k&=\rho L_{\mathcal G}x_k-B\omega_{k-1}. \label{eq:p-omega}
\end{align}
Equivalently,
\begin{align}
p_k&=\rho L_{\mathcal G}x_k+\lambda_k, \label{eq:p-lambda}\\
\lambda_{k+1}&=\lambda_k+\frac{\rho}{2}L_{\mathcal G}x_k.\label{eq:lambda-update}
\end{align}
Moreover, $\lambda_k\in\range(B)$, $\one^\top\lambda_k=0$ and $\one^\top p_k=0$.
\end{lemma}
\begin{IEEEproof}
See Appendix~\ref{app:message-privacy}.
\end{IEEEproof}
		
Two inputs are adjacent for client $i$ if $\mathcal D_i$ is replaced by $\mathcal D_i'$ of the same prescribed size $m_i$, with all other datasets unchanged. Let $\mathcal H_K$ be the augmented transcript of released states and exchanged messages through round $K-1$, together with the public initialization and schedules. Given the released-state history, all directed messages and auxiliary states are deterministic. Thus, privacy of $\mathcal H_K$ implies privacy of any observed message transcript.
For $\delta_i\in(0,1)$, define
\begin{align}
\rho_{i,\mathrm{priv}}:=\sum_{k=0}^{K-1} \frac{2(R_k^{\mathrm{clip}})^2}{\sigma_{\mathrm{dp},k}^2},
\epsilon_i:=\rho_{i,\mathrm{priv}} +2\sqrt{\rho_{i,\mathrm{priv}}\log(1/\delta_i)}. \notag
\end{align}
Applying adaptive composition to the Gaussian releases gives the following privacy guarantee for the full interactive transcript.
		\begin{theorem}[Transcript privacy]
			\label{thm:privacy}
			Use the public, data-independent initialization and schedules in Section~\ref{subsec:algorithm-design}. Suppose $\sigma_{\mathrm{dp},k}>0$ for $k=0,\ldots,K-1$. Then $\mathcal H_K$ is $\rho_{i,\mathrm{priv}}$-zCDP and $(\epsilon_i,\delta_i)$-DP with respect to replacement of $\mathcal D_i$.
		\end{theorem}
		\begin{IEEEproof}
			See Appendix~\ref{app:message-privacy}.
		\end{IEEEproof}
		
		Conditioning on $\mathcal H_k$ fixes the explicit graph correction. The clipped local increment may remain randomized, so the proof bounds the divergence between Gaussian mixtures before applying adaptive composition.
		
		\subsection{Network Stability and Disagreement}
		\label{subsec:stability-analysis}
		
		Throughout Sections~\ref{subsec:stability-analysis}--\ref{subsec:complexity-analysis}, impose Assumptions~\ref{ass:objective}--\ref{ass:oracle} and the common initialization in Section~\ref{subsec:algorithm-design}. Use synchronized constants $\tau_i=\tau\ge1$, $\eta_k=\eta>0$, and $\mu_k=\mu>0$, with $\rho,\beta>0$ fixed within each run. The clipping and noise schedules may vary with $k$.
		
		Typically, write $\bar\phi_k^t:=N^{-1}\sum_i\phi_{i,k}^t$ and $\bar x_k:=N^{-1}\sum_i x_{i,k}$. Define the disagreement measures as
		\[
		\mathcal C_k:=\frac{1}{N}\E[\|\Pi x_k\|^2],\qquad \mathcal C_{k,t}^{\phi} :=\frac{1}{N}\sum_{i=1}^N\E[\|\phi_{i,k}^t-\bar\phi_k^t\|^2].
		\]
		Then, the clipping and cumulative release energies are
		\begin{align*}
		\mathcal E_{\mathrm{cl},K}:=\sum_{k=0}^{K-1}E_k^{\mathrm{cl}},
		\mathcal E_{\mathrm{dp},K}:=\sum_{k=0}^{K-1}d\sigma_{\mathrm{dp},k}^2,
		\end{align*}
where $E_k^{\mathrm{cl}}:=\frac{1}{N}\sum_{i=1}^N \E[\|e_{i,k}^{\mathrm{cl}}\|^2]$.		
Thus, Lemma~\ref{lem:message} associates each nonzero Laplacian eigenvalue $\lambda$ with the homogeneous modal matrix
\[
M_\lambda=
\begin{bmatrix}
1-a\rho\lambda&-a\\
\rho\lambda/2&1
\end{bmatrix},
\]
where $\alpha=\eta\mu$ and $a=\tau\beta\alpha$.
		Local updates, clipping and release noise enter as additive inputs. Under the stability condition
		\begin{equation}
		0<a\rho\lambda_N<\frac{8}{3},
		\label{eq:stability}
		\end{equation}
		let $P_\lambda\succ0$ solve
		\begin{equation}
		M_\lambda^\top P_\lambda M_\lambda-P_\lambda=-I_2.
		\label{eq:lyapunov}
		\end{equation}
		Define
		\begin{align*}
		\underline p&:=\min_{2\le\ell\le N}\lambda_{\min}(P_{\lambda_\ell}),&
		\overline p&:=\max_{2\le\ell\le N}\lambda_{\max}(P_{\lambda_\ell}),
		\end{align*}
		and set $\chi:=1/(2\overline p)$ and $c_w:=\overline p(2\overline p-1)$. Let $U_\perp$ contain orthonormal eigenvectors for the nonzero Laplacian eigenvalues, and write
		\[
		\xi_k:=\operatorname{col}\!\left(
		(U_\perp^\top\!\otimes I_d)x_k,\,
		(U_\perp^\top\!\otimes I_d)\lambda_k\right).
		\]
		Let $\mathsf S$ reorder $\xi_k$ into $2d$-dimensional primal--dual pairs indexed by eigenmode. The Lyapunov matrix and energy are
		\[
		\mathcal P:=\mathsf S^\top \diag_{\ell=2}^N(P_{\lambda_\ell}\otimes I_d)\mathsf S,\qquad \mathcal V_k:=\frac{1}{N}\E[\xi_k^\top\mathcal P\xi_k].
		\]
This Lyapunov energy measures how network contraction competes with perturbations from local updates and private releases.		
		\begin{lemma}[Schur stability and input bound]
			\label{lem:stability}
			The matrices $M_\lambda$ are Schur stable for all nonzero graph modes if and only if \eqref{eq:stability} holds. Under this condition, \eqref{eq:lyapunov} has a unique positive-definite solution for each mode, and, for each round $k$,
			\begin{align}
			\mathcal V_{k+1}\le{}&(1-\chi)\mathcal V_k
			+c_w\!\Bigl[
			18\eta^2\tau^2Q^2+2E_k^{\mathrm{cl}}
			\notag\\
			&\quad
			+\left(1-\frac{1}{N}\right)d\sigma_{\mathrm{dp},k}^2
			\Bigr].
			\label{eq:V-recursion}
			\end{align}
			Furthermore,
			\begin{align}
			\mathcal C_k&\le\frac{\mathcal V_k}{\underline p},
			\label{eq:C-by-V}\\
			\frac{1}{N}\E\!\left[\|p_k\|^2\right]
			&\le
			\frac{2(\rho^2\lambda_N^2+1)}{\underline p}\mathcal V_k.
			\label{eq:p-by-V}
			\end{align}
			The common initialization gives $\mathcal V_0=0$.
		\end{lemma}
		\begin{IEEEproof}
			See Appendix~\ref{app:stability}.
		\end{IEEEproof}
		
		The modal condition governs homogeneous disagreement, while the oracle and descent bounds control the stochastic inputs. The single-client case follows by removing the network terms.
		
		\begin{remark}
The modal analysis describes coordination among $N\ge2$ clients. For $N=1$, the graph correction and dual state vanish, so $p_k=\lambda_k=0$. Both disagreement measures are identically zero because the network average equals the sole client's iterate. The algorithm then performs local derivative-free updates with clipping and Gaussian noise applied at each release. Its descent estimate follows directly from smoothness and the oracle bounds in Lemma~\ref{lem:oracle_memory}, without introducing graph eigenvalues or modal Lyapunov constants.
		\end{remark}
		
		To bound within-round drift, define
		\[
		\Gamma:=
		\frac{3+6(\tau-1)^2\alpha^2\beta^2(\rho^2\lambda_N^2+1)}
		{\underline p}.
		\]
The next estimate transfers control of the released-state disagreement to the local iterates between communication rounds.
		\begin{lemma}[Within-round disagreement]
			\label{lem:local_disagreement}
			Under \eqref{eq:stability}, for each round $k$ and $t=0,\ldots,\tau-1$,
			\begin{align}
			\mathcal C_{k,t}^{\phi}
			\le \Gamma\mathcal V_k+27\eta^2t^2Q^2.
			\label{eq:local-disagreement}
			\end{align}
		\end{lemma}
		\begin{IEEEproof}
			See Appendix~\ref{app:descent-main}.
		\end{IEEEproof}
		
		\subsection{Finite-Horizon Stationarity and Consensus}
		\label{subsec:nonconvex-analysis}
		
		For the descent bound, collect the local-query and smoothing terms in
		\begin{align*}
		A_{\mathrm{loc}}
		:={}&\frac{9L_f}{2}\tau\eta^2Q^2+\tau\alpha c_r^2\mu^2\\
		&+\frac{9}{8}\alpha L_f^2\eta^2Q^2
		\bigl[2\tau(\tau-1)(2\tau-1)+1\bigr].
		\end{align*}
Applying smoothness with the oracle and disagreement bounds gives the objective decrease over one communication round.
		\begin{lemma}[One-round descent]
			\label{lem:descent}
			Under the standing assumptions and \eqref{eq:stability}, for each round $k$,
			\begingroup
			\renewcommand{\mintagsep}{3pt}
			\begin{align}
			&\E[F(\bar x_{k+1})]
			\le \E[F(\bar x_k)]
			-\frac{\alpha}{8}\sum_{t=0}^{\tau-1}
			\E[\|\nabla F(\bar\phi_k^t)\|^2]\notag\\
			&\enspace+\frac{\alpha L_f^2\tau\Gamma}{2}\mathcal V_k
			+A_{\mathrm{loc}}
			+\left(\frac{4}{\alpha}+\frac{L_f}{2}\right)E_k^{\mathrm{cl}}
			+\frac{L_f}{2N}d\sigma_{\mathrm{dp},k}^2.
			\label{eq:round-descent}
			\end{align}
			\endgroup
		\end{lemma}
		\begin{IEEEproof}
			See Appendix~\ref{app:descent-main}.
		\end{IEEEproof}
		
		Define the time-averaged stationarity--consensus criterion
		\[
		\mathcal S_K :={}\frac{1}{K\tau} \sum_{k=0}^{K-1}\sum_{t=0}^{\tau-1} \E[\|\nabla F(\bar\phi_k^t)\|^2] +\frac{1}{K}\sum_{k=0}^{K-1}\mathcal C_k,
		\]
		and the average release energies
		\[
		\overline{\mathcal E}_{\mathrm{cl},K}:=\frac{\mathcal E_{\mathrm{cl},K}}{K},
		\qquad
		\overline{\mathcal E}_{\mathrm{dp},K}:=\frac{\mathcal E_{\mathrm{dp},K}}{K}.
		\]
Summing the descent inequality and controlling the accumulated disagreement yield the following bound on $\mathcal S_K$.
		\begin{theorem}[Stationarity and consensus]
			\label{thm:main}
			Suppose Assumptions~\ref{ass:objective}--\ref{ass:oracle} hold. Use the common initialization and the synchronized parameters in Section~\ref{subsec:stability-analysis}, with a fixed round gain $a=\tau\beta\eta\mu$ satisfying \eqref{eq:stability}. Then, for each $K\ge1$,
			\begin{align}
			&\mathcal S_K=\mathcal O\!\Biggl(
			\frac{1}{K\tau\alpha}+\frac{\eta}{\mu}+\mu^2+\eta^2\tau^2
			\notag\\
			&\quad+\left(1+\frac{1}{\tau\alpha^2}\right)
			\overline{\mathcal E}_{\mathrm{cl},K}
			+\left(1+\frac{1}{N\tau\alpha}\right)
			\overline{\mathcal E}_{\mathrm{dp},K}\Biggr).
			\label{eq:main-rate}
			\end{align}
			The implicit constant depends only on the initial objective gap, $\rho$, $a$, the fixed graph, and the smoothness and oracle constants. It is independent of $K$, $\tau$, $\eta$, $\mu$, and the release energies.
		\end{theorem}
		\begin{IEEEproof}
			See Appendix~\ref{app:descent-main}.
		\end{IEEEproof}
		
		The first four terms in \eqref{eq:main-rate} account for derivative-free descent and local drift; the remaining terms quantify clipping and Gaussian release. These residual terms do not establish an unavoidable error floor. The bound is time-averaged, not a last-iterate guarantee; the internal iterates used in $\mathcal S_K$ are not additional releases.
		
		For the non-private case, set $e_{i,k}^{\mathrm{cl}}=0$ and $\sigma_{\mathrm{dp},k}=0$, and write $T:=K\tau$. For fixed $\eta_0,\mu_0>0$ and $0<a_0<8/(3\rho\lambda_N)$, choose
		\begin{equation}
		\eta_T=\eta_0T^{-1/2},
		\quad
		\mu_T=\mu_0T^{-1/6},
		\quad
		\beta_{T,\tau}
		=
		\frac{a_0}{\tau\eta_T\mu_T}.
		\label{eq:nonprivate-schedule}
		\end{equation}
Substituting this schedule into Theorem~\ref{thm:main} quantifies how the local stage length affects the rate.
		\begin{corollary}[Non-private local-update rate]
			\label{cor:nonprivate}
			Under the conditions of Theorem~\ref{thm:main}, use \eqref{eq:nonprivate-schedule} with zero clipping error and release noise. If the oracle moment bound is uniform over this family of runs, then
			\begin{equation}
			\mathcal S_K
			=
			\mathcal O\!\left(
			T^{-1/3}
			+\frac{\tau^2}{T}
			\right).
			\label{eq:nonprivate-rate}
			\end{equation}
			In particular, $\tau=\mathcal O(T^{1/3})$ preserves the $\mathcal O(T^{-1/3})$ order, and $\tau=\Theta(T^{1/3})$ gives $K=\Theta(T^{2/3})$.
		\end{corollary}
		\begin{IEEEproof}
			See Appendix~\ref{app:descent-main}.
		\end{IEEEproof}
		
		Under this schedule, the round gain remains fixed while the individual parameters vary. The following remark addresses convex objectives.
		\begin{remark}
For convex objectives, Theorem~\ref{thm:main} still bounds stationarity and disagreement under the stated assumptions. If a minimizer $x^\star$ exists, convexity gives $F(x)-F(x^\star)\le\langle\nabla F(x),x-x^\star\rangle$, so exact stationarity implies global optimality. Turning this relation into an objective-gap rate requires additional control of the distance to the solution set. In particular, Corollary~\ref{cor:nonprivate} preserves the stationarity and consensus rate with the same reduction in communication rounds for convex instances.
		\end{remark}
		
		\subsection{Communication and Query Complexity}
		\label{subsec:complexity-analysis}
		
		Under the non-private schedule of Corollary~\ref{cor:nonprivate}, choosing $\tau=\Theta(T^{1/3})$ preserves the $\mathcal O(T^{-1/3})$ stationarity--consensus rate. For a tolerance $\varepsilon_{\mathrm{stat}}>0$ on $\mathcal S_K$, sufficient budgets are $T=\mathcal O(\varepsilon_{\mathrm{stat}}^{-3})$ local updates and $K=\mathcal O(\varepsilon_{\mathrm{stat}}^{-2})$ communication rounds. The latter improves on the sufficient $\mathcal O(\varepsilon_{\mathrm{stat}}^{-3})$ round bound for the $\tau=1$ specialization. Client $i$ uses $m_i+\tau b_i$ function queries per round, including memory initialization, giving $\mathcal Q_i=Km_i+Tb_i=\mathcal O(m_i\varepsilon_{\mathrm{stat}}^{-2}+b_i\varepsilon_{\mathrm{stat}}^{-3})$. Each directed edge carries one $d$-dimensional vector per round, so the total scalar communication cost is $2|\mathcal E|dK=\mathcal O(|\mathcal E|d\varepsilon_{\mathrm{stat}}^{-2})$. The memory table occupies $\mathcal O(m_id)$ storage in addition to the model and edge states. For private training with a constant clipping radius and uniform release-noise variance, Theorem~\ref{thm:privacy} gives $\sigma_{\mathrm{dp}}^2\propto K$ at a fixed total zCDP budget. At fixed $T$, increasing $\tau$ reduces both the number of releases $K=T/\tau$ and the noise variance required per release. Theorem~\ref{thm:main} quantifies the accompanying local-drift and release-error terms, relating these communication savings to the stationarity--consensus bound.
		
		The constants in this comparison remain controlled as the local-stage length grows. For fixed $\rho$ and graph, \eqref{eq:nonprivate-schedule} maintains $a=a_0$, so the modal matrices $M_\lambda$ and their Lyapunov solutions $P_\lambda$ are unchanged. Since $(\tau-1)^2\alpha^2\beta^2\le a_0^2$, the coefficient $\Gamma$ is uniformly bounded in $T$ and $\tau$. The uniform oracle-moment condition in Corollary~\ref{cor:nonprivate} also makes the bound $9Q^2$ in Lemma~\ref{lem:oracle_memory}(ii) independent of $T$ and $\tau$. Applying \eqref{eq:nonprivate-schedule} to Theorem~\ref{thm:main} gives \eqref{eq:nonprivate-rate} with uniform constants and the explicit local-drift term $\tau^2/T$. The privacy calculation uses the same round-level description. Each client forms its outgoing messages by deterministic post-processing of one protected state and the preceding transcript, as established in Theorem~\ref{thm:privacy}. Under replacement of one client's dataset, privacy composes over that client's $K$ releases, whereas the communication count includes all $2|\mathcal E|K$ directed messages. The other clients' conditional response laws are unchanged under this replacement, so adaptive composition covers the full interactive transcript.
	\end{spadecompact}

\section{Experimental Evaluation}
\label{sec:experiments}

\begingroup

\begingroup
We evaluate SPADE-DFL on four classification tasks under heterogeneous client-data partitions. We also examine the effects of client count and privacy budget, and compare communication costs under matched local-update budgets on a bounded nonconvex problem.

\subsection{Datasets and Models}

\begin{figure}[t]
  \centering
  \setlength{\tabcolsep}{0.1pt}

  \newcommand{\dsimg}[1]{%
    \includegraphics[
      width=0.098\columnwidth,
      height=0.12\columnwidth
    ]{#1}%
  }

  \begin{tabular}{@{}*{10}{c}@{}}

  \dsimg{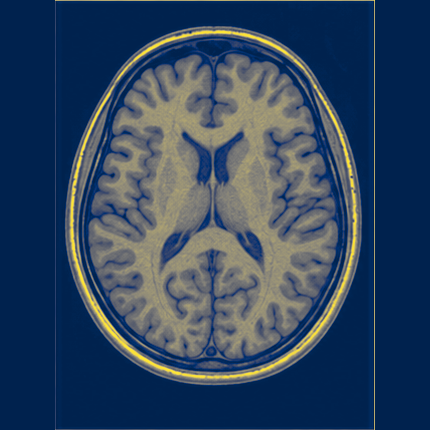}
  & \dsimg{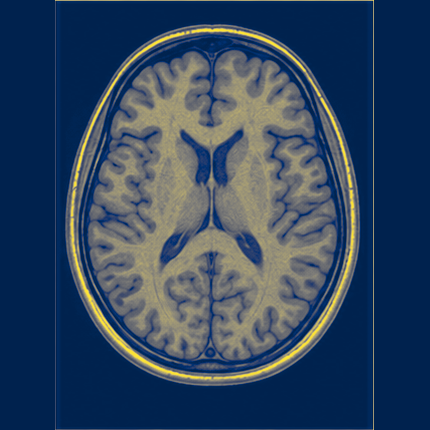}
  & \dsimg{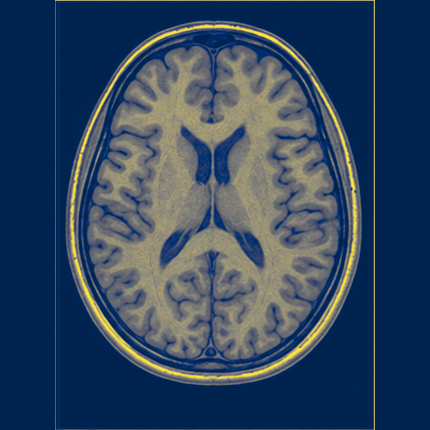}
  & \dsimg{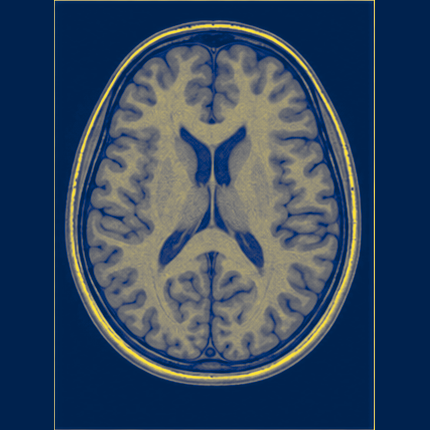}
  & \dsimg{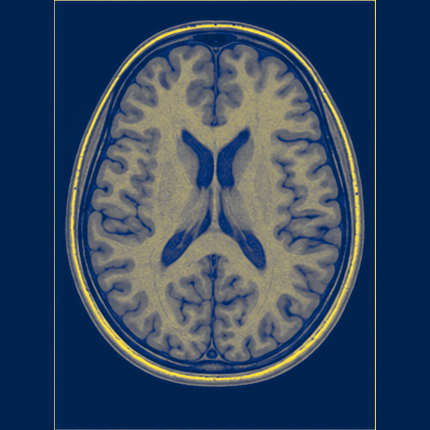}
  & \dsimg{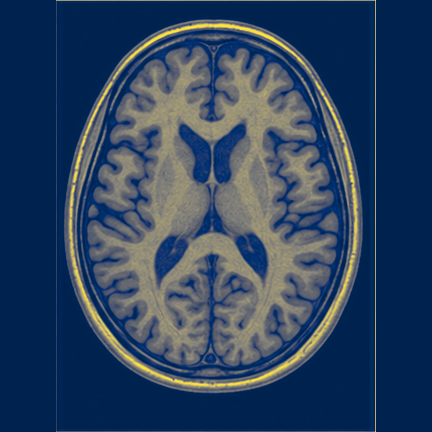}
  & \dsimg{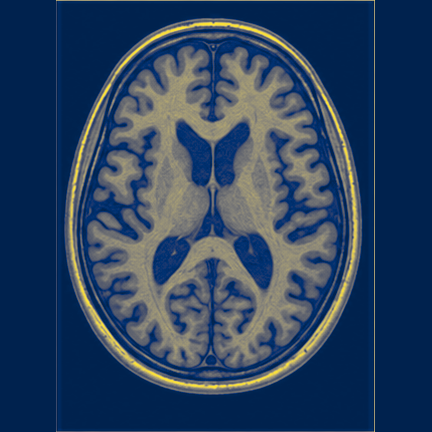}
  & \dsimg{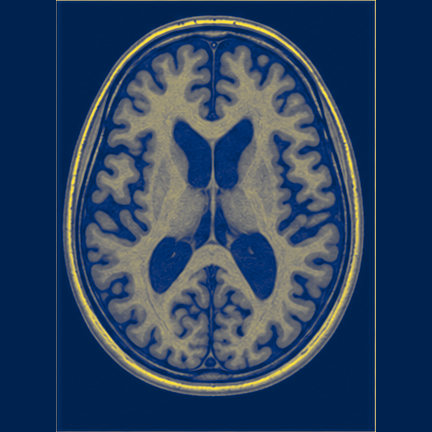}
  & \dsimg{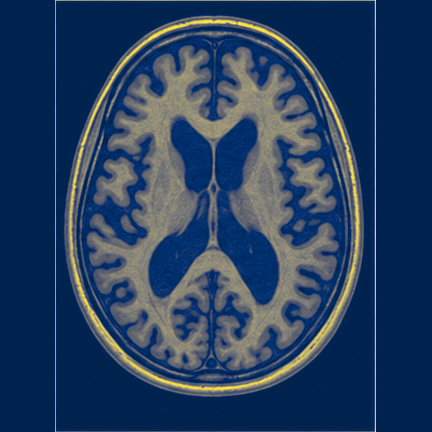}
  & \dsimg{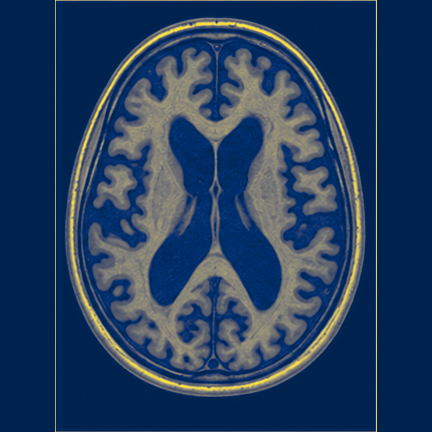}
  \\[-0.3pt]
  \specialrule{0.2pt}{-1pt}{1pt}

  \dsimg{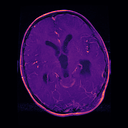}
  & \dsimg{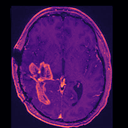}
  & \dsimg{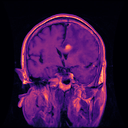}
  & \dsimg{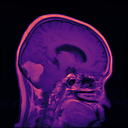}
  & \dsimg{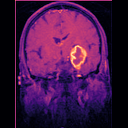}
  & \dsimg{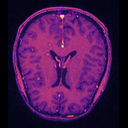}
  & \dsimg{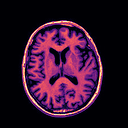}
  & \dsimg{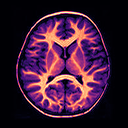}
  & \dsimg{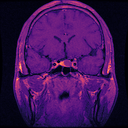}
  & \dsimg{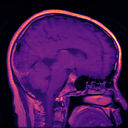}
  \\[-0.3pt]
  \specialrule{0.2pt}{-2pt}{0pt}

  \end{tabular}

  \vspace{-0.5mm}
  \caption{Illustrative brain images related to Alzheimer’s disease and brain tumors.}
  \label{fig:dataset-samples}
\end{figure}

We use four datasets, i.e., MNIST\datasetnote{https://www.kaggle.com/datasets/hojjatk/mnist-dataset}, Fashion-MNIST\datasetnote{https://www.kaggle.com/datasets/zalando-research/fashionmnist}, Alzheimer’s disease\datasetnote{https://www.kaggle.com/datasets/rabieelkharoua/alzheimers-disease-dataset}, and Brain Tumor MRI\datasetnote{https://www.kaggle.com/datasets/masoudnickparvar/brain-tumor-mri-dataset}.
The MNIST task distinguishes digits 6 and 7, and the Fashion-MNIST task distinguishes T-shirt/top from Trouser. The Alzheimer’s disease task uses structured clinical features for binary classification. Brain Tumor MRI is used for four-class image classification.

\begin{figure*}[t]
\color{black}
	\centering
        \begin{minipage}[t]{0.24\textwidth}
		\centering
		\includegraphics[width=\linewidth]{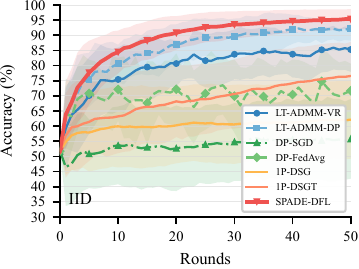}
	\end{minipage}\hfill
        \begin{minipage}[t]{0.24\textwidth}
		\centering
		\includegraphics[width=\linewidth]{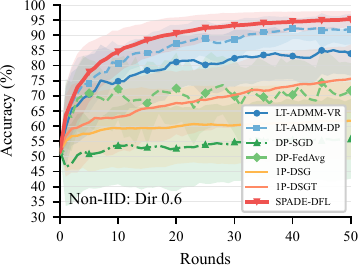}
	\end{minipage}\hfill
        \begin{minipage}[t]{0.24\textwidth}
		\centering
		\includegraphics[width=\linewidth]{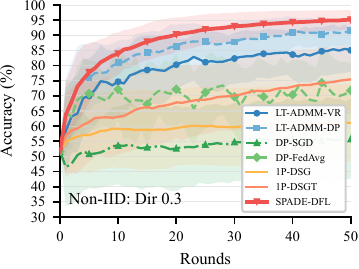}
	\end{minipage}\hfill
        \begin{minipage}[t]{0.24\textwidth}
		\centering
		\includegraphics[width=\linewidth]{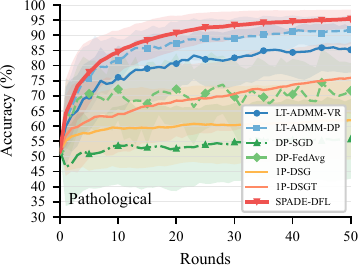}
        \end{minipage}
	\caption{Test accuracy on binary MNIST.}
	\label{fig:MNIST}
	\vspace{0.2cm}
	
        \begin{minipage}[t]{0.24\textwidth}
		\centering
		\includegraphics[width=\linewidth]{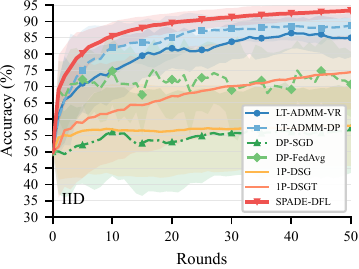}
	\end{minipage}\hfill
        \begin{minipage}[t]{0.24\textwidth}
		\centering
		\includegraphics[width=\linewidth]{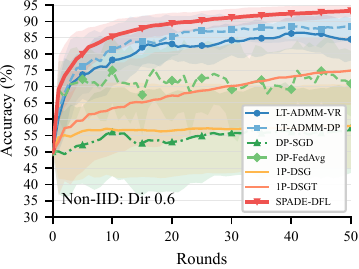}
	\end{minipage}\hfill
        \begin{minipage}[t]{0.24\textwidth}
		\centering
		\includegraphics[width=\linewidth]{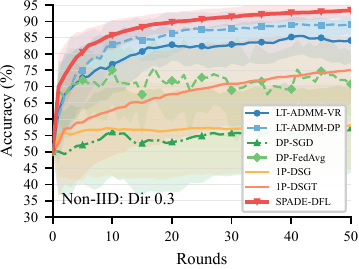}
	\end{minipage}\hfill
        \begin{minipage}[t]{0.24\textwidth}
		\centering
		\includegraphics[width=\linewidth]{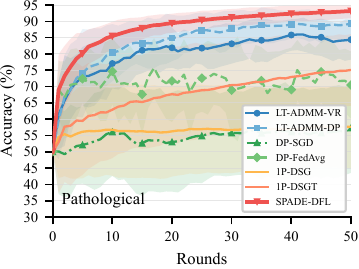}
        \end{minipage}
	\caption{Test accuracy on binary Fashion-MNIST.}
	\label{fig:Fashion MNIST}
	\vspace{0.2cm}
	
        \begin{minipage}[t]{0.24\textwidth}
		\centering
		\includegraphics[width=\linewidth]{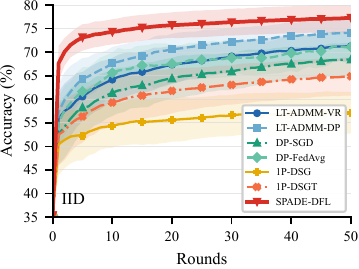}
	\end{minipage}\hfill
        \begin{minipage}[t]{0.24\textwidth}
		\centering
		\includegraphics[width=\linewidth]{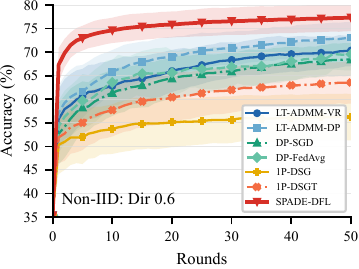}
	\end{minipage}\hfill
        \begin{minipage}[t]{0.24\textwidth}
		\centering
		\includegraphics[width=\linewidth]{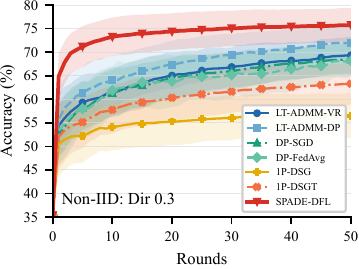}
	\end{minipage}\hfill
        \begin{minipage}[t]{0.24\textwidth}
		\centering
		\includegraphics[width=\linewidth]{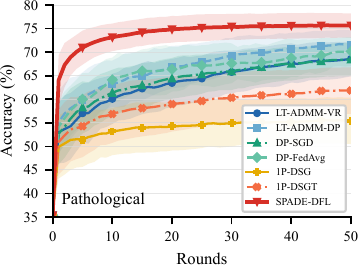}
        \end{minipage}
	\caption{Test accuracy on Alzheimer’s disease.}
    \label{fig:Alzheimer}
	\vspace{0.2cm}
	
        \begin{minipage}[t]{0.24\textwidth}
		\centering
		\includegraphics[width=\linewidth]{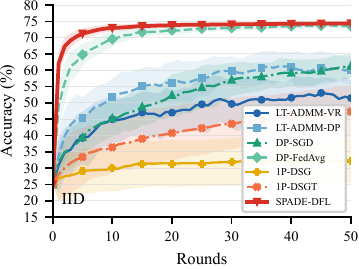}
	\end{minipage}\hfill
        \begin{minipage}[t]{0.24\textwidth}
		\centering
		\includegraphics[width=\linewidth]{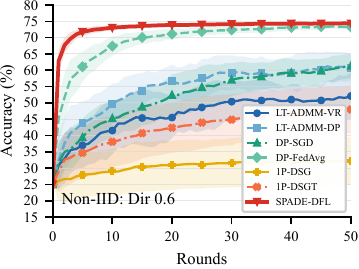}
	\end{minipage}\hfill
        \begin{minipage}[t]{0.24\textwidth}
		\centering
		\includegraphics[width=\linewidth]{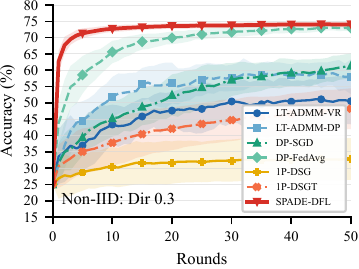}
	\end{minipage}\hfill
        \begin{minipage}[t]{0.24\textwidth}
		\centering
		\includegraphics[width=\linewidth]{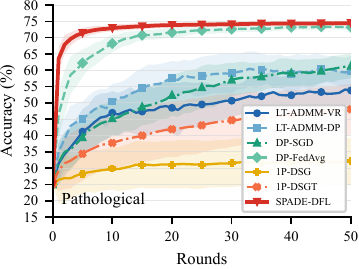}
        \end{minipage}
	\caption{Test accuracy on Brain Tumor MRI.}
	\label{fig:Brain Tumor}
\end{figure*}

For MNIST and Fashion-MNIST, the original 784-dimensional image vectors are reduced to 10 principal components before training. The Alzheimer’s disease task uses 39-dimensional structured clinical features. For Brain Tumor MRI, a frozen ResNet-18 V1 network extracts 512-dimensional representations. Principal component analysis fitted on the training split reduces these representations to 10 dimensions, followed by a four-class linear softmax classifier with 40 trainable parameters.

Four client-data distributions are considered: independent and identically distributed (IID), Dirichlet partitions with concentration parameters $\alpha=0.6$ and $\alpha=0.3$, and a pathological partition. A smaller Dirichlet concentration produces stronger variation in class proportions across clients. The same data partition and random seed are used across methods within each comparison.

\subsection{Comparative Evaluation under Heterogeneous Client Data}\label{VB}
We compare SPADE-DFL with LT-ADMM-VR, LT-ADMM-DP, DP-SGD, DP-FedAvg, 1P-DSG, and 1P-DSGT. The LT-ADMM-VR, LT-ADMM-DP, DP-SGD, and DP-FedAvg implementations use symmetric two-point component-loss estimates, while 1P-DSG, 1P-DSGT, and SPADE-DFL use one-point estimates. Function-query budgets and privacy settings vary across methods.

The comparison uses 50 rounds and 30 paired random seeds. The decentralized implementations use \(N=31\) clients on an undirected Erdős–Rényi graph with edge probability 0.3. The centralized DP-SGD baseline uses pooled data with \(N=1\), so its results are shared across the client-data partitions. SPADE-DFL uses $\tau=2$ local updates for MNIST, Fashion-MNIST, and Brain Tumor MRI, and \(\tau=8\) for Alzheimer’s disease.
For the comparative experiments, SPADE-DFL reinitializes the estimator memory at the beginning of each round
using~(4) and is evaluated without release noise.

For MNIST, Fashion-MNIST, and Brain Tumor MRI, the curves and final-round values are summarized over 30 random seeds using the arithmetic mean and population standard deviation. For Alzheimer’s disease, the same statistics are calculated over 150 fold--seed trajectories obtained from five stratified folds and 30 seeds per fold. All accuracy values are reported as percentages.
Figs.~\ref{fig:MNIST}--\ref{fig:Brain Tumor} present the accuracy trajectories over 50 communication rounds. In each figure, the panels from left to right correspond to IID, Dirichlet $\alpha=0.6$, Dirichlet $\alpha=0.3$, and pathological partitions. The curves show the arithmetic mean, and the shaded regions show $\pm$ one population standard deviation. Table~\ref{tab:distribution_final_accuracy} reports the corresponding final-round accuracies.
Boldface marks the largest mean in each row, underlining marks the largest competing mean, and $\Delta_{\mathrm{best}}$ denotes their difference in percentage points.
\endgroup

\begin{table*}[htbp!]
\centering
\caption{Test accuracy at round 50 under the configurations in Section \ref{VB}.}
\label{tab:distribution_final_accuracy}
\scriptsize
\setlength{\tabcolsep}{2.2pt}
\renewcommand{\arraystretch}{0.90}
\begin{tabular*}{\textwidth}{@{\extracolsep{\fill}}llcccccccc@{}}
\toprule
Dataset & Distribution
& LT-ADMM-VR
& LT-ADMM-DP
& DP-SGD
& DP-FedAvg
& 1P-DSG
& 1P-DSGT
& SPADE-DFL
& $\Delta_{\mathrm{best}}$ \\
\midrule
\multirow{4}{*}{MNIST}
& IID
& $85.30{\pm}7.35$
& $\underline{92.22{\pm}3.97}$
& $55.81{\pm}13.16$
& $71.69{\pm}8.96$
& $62.15{\pm}12.91$
& $76.58{\pm}11.20$
& $\mathbf{95.40{\pm}3.24}$
& $+3.18$ \\
& Dir. ($\alpha=0.6$)
& $83.95{\pm}7.29$
& $\underline{92.04{\pm}5.01}$
& $55.81{\pm}13.16$
& $71.68{\pm}9.11$
& $61.81{\pm}12.81$
& $75.68{\pm}11.33$
& $\mathbf{95.31{\pm}2.68}$
& $+3.27$ \\
& Dir. ($\alpha=0.3$)
& $85.07{\pm}6.97$
& $\underline{91.55{\pm}4.75}$
& $55.81{\pm}13.16$
& $71.75{\pm}9.15$
& $61.15{\pm}12.78$
& $75.45{\pm}11.17$
& $\mathbf{95.15{\pm}3.13}$
& $+3.60$ \\
& Pathological
& $85.21{\pm}7.48$
& $\underline{92.05{\pm}4.85}$
& $55.81{\pm}13.16$
& $71.69{\pm}9.11$
& $61.98{\pm}12.97$
& $75.97{\pm}11.81$
& $\mathbf{95.37{\pm}3.08}$
& $+3.32$ \\
\midrule
\multirow{4}{*}{Fashion-MNIST}
& IID
& $84.95{\pm}5.94$
& $\underline{88.74{\pm}4.03}$
& $57.33{\pm}13.79$
& $70.67{\pm}9.98$
& $58.01{\pm}12.37$
& $74.43{\pm}11.36$
& $\mathbf{93.28{\pm}2.21}$
& $+4.54$ \\
& Dir. ($\alpha=0.6$)
& $84.48{\pm}6.51$
& $\underline{88.42{\pm}4.68}$
& $57.33{\pm}13.79$
& $70.89{\pm}10.04$
& $57.90{\pm}13.03$
& $74.92{\pm}10.46$
& $\mathbf{93.25{\pm}1.87}$
& $+4.83$ \\
& Dir. ($\alpha=0.3$)
& $84.12{\pm}6.86$
& $\underline{89.13{\pm}3.22}$
& $57.33{\pm}13.79$
& $70.77{\pm}10.18$
& $58.03{\pm}12.94$
& $75.09{\pm}10.63$
& $\mathbf{93.36{\pm}2.18}$
& $+4.23$ \\
& Pathological
& $84.39{\pm}6.13$
& $\underline{89.32{\pm}4.29}$
& $57.33{\pm}13.79$
& $70.49{\pm}10.30$
& $57.69{\pm}12.61$
& $75.11{\pm}10.99$
& $\mathbf{93.17{\pm}2.22}$
& $+3.85$ \\
\midrule
\multirow{4}{*}{Alzheimer’s disease}
& IID
& $71.28{\pm}3.18$
& $\underline{73.99{\pm}2.64}$
& $68.47{\pm}3.44$
& $71.10{\pm}3.10$
& $57.11{\pm}4.41$
& $64.86{\pm}4.08$
& $\mathbf{77.27{\pm}2.91}$
& $+3.28$ \\
& Dir. ($\alpha=0.6$)
& $70.34{\pm}3.32$
& $\underline{73.10{\pm}3.24}$
& $68.47{\pm}3.44$
& $69.56{\pm}3.43$
& $56.20{\pm}4.95$
& $63.56{\pm}4.49$
& $\mathbf{77.31{\pm}2.86}$
& $+4.21$ \\
& Dir. ($\alpha=0.3$)
& $69.33{\pm}3.93$
& $\underline{72.02{\pm}3.60}$
& $68.47{\pm}3.44$
& $68.18{\pm}3.51$
& $56.49{\pm}4.75$
& $63.27{\pm}4.54$
& $\mathbf{75.73{\pm}3.62}$
& $+3.71$ \\
& Pathological
& $68.56{\pm}3.72$
& $\underline{71.61{\pm}3.02}$
& $68.47{\pm}3.44$
& $70.17{\pm}3.05$
& $55.42{\pm}4.93$
& $61.84{\pm}4.95$
& $\mathbf{75.64{\pm}2.58}$
& $+4.03$ \\
\midrule
\multirow{4}{*}{Brain Tumor MRI}
& IID
& $51.46{\pm}5.42$
& $60.06{\pm}4.47$
& $61.40{\pm}3.79$
& $\underline{73.52{\pm}0.75}$
& $32.28{\pm}6.81$
& $47.35{\pm}7.86$
& $\mathbf{74.38{\pm}0.52}$
& $+0.86$ \\
& Dir. ($\alpha=0.6$)
& $52.09{\pm}7.11$
& $60.81{\pm}4.06$
& $61.40{\pm}3.79$
& $\underline{73.28{\pm}0.79}$
& $32.23{\pm}7.24$
& $47.95{\pm}7.35$
& $\mathbf{74.35{\pm}0.55}$
& $+1.07$ \\
& Dir. ($\alpha=0.3$)
& $50.50{\pm}7.08$
& $58.01{\pm}5.35$
& $61.40{\pm}3.79$
& $\underline{72.96{\pm}0.92}$
& $32.81{\pm}6.40$
& $48.22{\pm}6.26$
& $\mathbf{74.07{\pm}0.80}$
& $+1.11$ \\
& Pathological
& $53.81{\pm}6.08$
& $59.61{\pm}4.86$
& $61.40{\pm}3.79$
& $\underline{73.20{\pm}0.87}$
& $32.03{\pm}7.24$
& $48.06{\pm}7.70$
& $\mathbf{74.44{\pm}0.63}$
& $+1.24$ \\
\bottomrule
\end{tabular*}
\end{table*}

Under the stated oracle and privacy configurations, SPADE-DFL achieved the highest mean test accuracy at round 50 in all 16 dataset--partition combinations reported in Table \ref{tab:distribution_final_accuracy}.
The gains over the best competing method ranged from 0.86 percentage points on IID Brain Tumor MRI to 4.83 percentage points on Fashion-MNIST with Dirichlet $\alpha=0.6$. Across the four partitions, SPADE-DFL achieved mean accuracies of 95.15–95.40\% on MNIST, 93.17–93.36\% on Fashion-MNIST, 75.64–77.31\% on Alzheimer’s disease, and 74.07–74.44\% on Brain Tumor MRI.

\subsection{Sensitivity to Client Count and Privacy Budget}\label{Sensitivity to Client Count and Privacy Budget}

\begin{table}[htbp!]
\centering

\caption{Validation accuracy under varying client counts, $\varepsilon=32$.}
\label{tab:client_count_accuracy}

\scriptsize
\setlength{\tabcolsep}{1.5pt}
\setlength{\medmuskip}{0mu}
\renewcommand{\arraystretch}{0.95}

\resizebox{\columnwidth}{!}{%
\begin{tabular}{@{}lccccc@{}}
\toprule
Dataset & $N=30$ & $N=40$ & $N=50$ & $N=75$ & $N=100$ \\
\midrule
MNIST
& $99.21\pm0.15$ & $99.03\pm0.15$ & $99.21\pm0.19$
& $99.21\pm0.23$ & $99.15\pm0.11$ \\
Fashion-MNIST
& $94.22\pm0.67$ & $92.28\pm0.98$ & $92.23\pm1.03$
& $93.00\pm0.90$ & $93.55\pm0.70$ \\
Alzheimer’s disease
& $72.16\pm0.79$ & $72.16\pm1.43$ & $72.84\pm0.89$
& $72.49\pm1.06$ & $73.00\pm0.48$ \\
Brain Tumor MRI
& $65.46\pm3.72$ & $68.82\pm3.56$ & $65.00\pm2.34$
& $70.46\pm3.02$ & $70.96\pm3.40$ \\
\bottomrule
\end{tabular}%
}
\end{table}

\begin{table}[htbp!]
\centering
\caption{Validation accuracy under varying privacy budgets, $N=100$.}
\label{tab:privacy_budget_accuracy}

\scriptsize
\setlength{\tabcolsep}{1.5pt}
\setlength{\medmuskip}{0mu}
\renewcommand{\arraystretch}{0.95}

\resizebox{\columnwidth}{!}{%
\begin{tabular}{@{}lccccc@{}}
\toprule
Dataset
& $\varepsilon=4$ & $\varepsilon=8$ & $\varepsilon=16$
& $\varepsilon=24$ & $\varepsilon=32$ \\
\midrule
MNIST
& $93.32\pm8.71$ & $97.98\pm1.79$ & $99.05\pm0.12$
& $99.15\pm0.16$ & $99.15\pm0.11$ \\
Fashion-MNIST
& $92.50\pm2.18$ & $93.38\pm1.81$ & $93.43\pm1.01$
& $93.58\pm0.88$ & $93.55\pm0.70$ \\
Alzheimer’s disease
& $61.05\pm3.17$ & $67.26\pm2.55$ & $71.56\pm0.74$
& $72.81\pm0.76$ & $73.00\pm0.48$ \\
Brain Tumor MRI
& $43.29\pm9.92$ & $54.00\pm6.73$ & $64.18\pm5.50$
& $69.25\pm4.60$ & $70.96\pm3.40$ \\
\bottomrule
\end{tabular}%
}

\end{table}


Tables~\ref{tab:client_count_accuracy} and~\ref{tab:privacy_budget_accuracy} report validation accuracy for SPADE-DFL under Dirichlet partitions with $\alpha=0.3$; MNIST, Fashion-MNIST, and Alzheimer's disease use balanced-capacity partitions, whereas Brain Tumor MRI uses the ordinary Dirichlet partition recorded in the adopted runs.
The reported privacy budgets use client-level replacement adjacency with \(\delta=10^{-5}\), where adjacent inputs replace one client's entire fixed-size dataset while leaving all other clients' fixed preprocessed datasets unchanged; the accountant composes one Gaussian release per client per round with sensitivity \(2R_k^{\mathrm{clip}}\), and the guarantee is conditional on the frozen preprocessing and partition rather than end-to-end from raw data.

Table~\ref{tab:client_count_accuracy} varies the number of clients from 30 to 100 at $\varepsilon=32$. MNIST accuracy remained between 99.03\% and 99.21\%. The corresponding ranges were 92.23–94.22\% for Fashion-MNIST, 72.16–73.00\% for Alzheimer’s disease, and 65.00–70.96\% for Brain Tumor MRI. Brain Tumor MRI showed the largest variation and reached its highest mean accuracy at $N=100$.

Table~\ref{tab:privacy_budget_accuracy} varies $\varepsilon$ from 4 to 32 at $N=100$. Mean validation accuracy increased from 93.32\% to 99.15\% on MNIST, from 92.50\% to 93.55\% on Fashion-MNIST, from 61.05\% to 73.00\% on Alzheimer’s disease, and from 43.29\% to 70.96\% on Brain Tumor MRI. The largest gain occurred on Brain Tumor MRI. MNIST accuracy changed little beyond $\varepsilon=16$, while Fashion-MNIST varied by approximately one percentage point over the tested range.

\subsection{Client Heterogeneity and Roundwise Comparisons}

\begin{figure}[!t]
\color{black}
	\centering
        \includegraphics[width=\columnwidth]{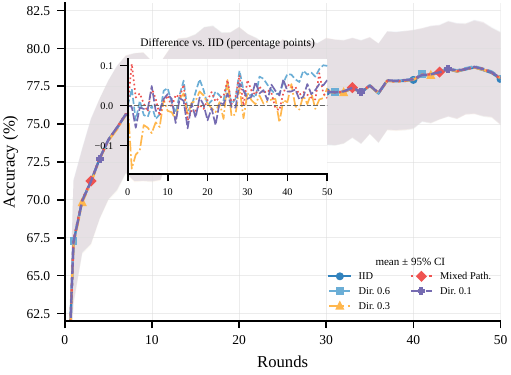}
	\caption{Test accuracy across five MNIST partitions.}
	\label{distribution-accuracy}
\end{figure}

\begin{figure}[htbp!]
\color{black}
	\centering
        \includegraphics[width=\columnwidth]{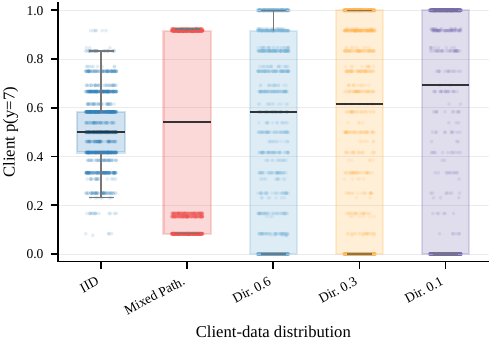}
	\caption{Client-level digit-7 proportions across MNIST partitions.}
	\label{client-label-heterogeneity}
\end{figure}

\begin{figure}[!t]
\color{black}
\centering
\includegraphics[width=\columnwidth]{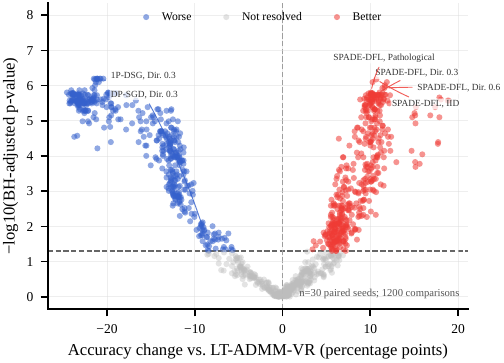}
\caption{Accuracy differences relative to LT-ADMM-VR.}
\label{method-volcano}
\end{figure}

A separate heterogeneity experiment evaluates SPADE-DFL on binary MNIST using $N=1000$ clients on a degree-two ring, $\tau=8$, and $(\varepsilon,\delta)=(8,10^{-5})$. The experiment uses 50 communication rounds and 30 paired random seeds. In Fig.~\ref{distribution-accuracy}, the curves show the arithmetic mean, and the shaded regions show two-sided 95\% Student's $t$ confidence intervals. The inset reports the matched-seed mean accuracy difference relative to the IID partition.
Fig.~\ref{client-label-heterogeneity} shows the client-level proportion of digit 7 under IID, pathological denoted by Mixed Path, and Dirichlet partitions with $\alpha\in\{0.6,0.3,0.1\}$.
The mean accuracy curves in Fig.~\ref{distribution-accuracy} are closely aligned across the five partitions. Fig.~\ref{client-label-heterogeneity} shows greater variation in the proportion of digit-7 samples across clients under the non-IID partitions.

In Fig.~\ref{method-volcano}, each point represents one method--partition--round comparison relative to LT-ADMM-VR. The horizontal coordinate is the paired mean accuracy difference over 30 seeds. Statistical comparisons use two-sided paired Wilcoxon signed-rank tests, with Benjamini--Hochberg correction across the 1,200 comparisons. These tests provide comparisons along the training trajectories, whose successive rounds are statistically dependent.

\subsection{Communication--Optimization Trade-off}

\begin{figure}[!t]
\centering
\includegraphics[width=\columnwidth]{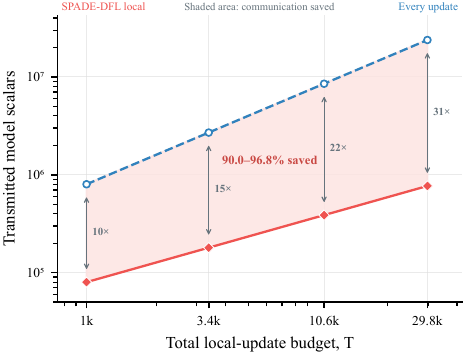}
\caption{Communication cost under matched local-update budgets.}
\label{fig_communication_saving_band}
\end{figure}

Fig.~\ref{fig_communication_saving_band} compares communication after every local update, corresponding to $\tau=1$, with the prescribed local-update schedule $\tau=T^{1/3}$. The experiment uses $N=20$ clients arranged in a cycle, model dimension $d=20$, and total local-update budgets $T\in\{1000,3375,10648,29791\}$. These budgets correspond to $\tau\in\{10,15,22,31\}$ under the local-update schedule.

Cumulative communication is measured as $2|\mathcal{E}|dK$ transmitted model scalars, where $K=T/\tau$. The local-update schedule reduced communication by factors of 10, 15, 22, and 31 for the four update budgets. At \(T=29{,}791\), the communication volume was 768,800 scalars with $\tau=31$ and 23,832,800 scalars with $\tau=1$, a reduction of 96.8\%.

The empirical stationarity–consensus criterion decreased with the update budget under both schedules. At $T=29{,}791$, its mean over 30 paired seeds was 0.010392 with $\tau=31$ and 0.000372 with $\tau=1$.
At this budget, the 96.8\% reduction in transmitted scalars is accompanied by a larger stationarity--consensus residual.
\par
\endgroup

\section{Conclusion}
SPADE-DFL characterizes how local computation can replace neighbor communication when learning relies on single-point function evaluations. For smooth nonconvex objectives under uniform query-moment bounds, the prescribed nonprivate schedule yields a time-averaged stationarity and consensus bound of $\mathcal{O}(T^{-1/3}+\tau^2/T)$, where $T$ is the number of local updates per client and $\tau$ is the number of updates between exchanges. This dependence permits $\tau=\Theta(T^{1/3})$, reducing communication to $\Theta(T^{2/3})$ rounds while preserving the $\mathcal{O}(T^{-1/3})$ convergence order. The communication interval also determines how often local information is privately released. Protecting the accumulated data-dependent increment once per round establishes client-level differential privacy for the full interactive transcript. With a fixed clipping radius and total zCDP budget, fewer exchanges reduce the required uniform Gaussian noise variance per release. The finite-horizon bound quantifies the accompanying local drift, clipping, and release errors, making explicit how the choice of communication interval affects the optimization cost of private training.
	
	\begin{spadecompact}
		\appendices
		
		\section{Oracle and Memory Bounds}
		\label{app:oracle-memory}
		
		\begin{IEEEproof}[Proof of Lemma~\ref{lem:oracle_memory}]
			For (i), apply smoothness at the $\mathcal F$-measurable point $x$:
			\begin{align}
			f_{i,h}(x+\mu u)
			&=f_{i,h}(x)+\mu\langle\nabla f_{i,h}(x),u\rangle+R(x,u,\mu),
			\label{eq:taylor-expansion}\\
			|R(x,u,\mu)|
			&\le\frac{L_f\mu^2}{2}\|u\|^2.
			\label{eq:taylor-remainder}
			\end{align}
			Assumption~\ref{ass:oracle} gives $\E[u\mid\mathcal F]=0$, $\E[uu^\top\mid\mathcal F]=c_uI_d$, and $\E[u\zeta\mid\mathcal F]=\E[u\,\E[\zeta\mid\mathcal F,u]\mid\mathcal F]=0$. Hence the constant and noise terms vanish, while the linear term yields $\mu\nabla f_{i,h}(x)$. The remainder satisfies
			\begin{align*}
			r_{i,h}(x,\mu)
			&=\frac{1}{c_u}\E[uR(x,u,\mu)\mid\mathcal F],\\
			\|r_{i,h}(x,\mu)\|
			&\le\frac{L_f\mu^2}{2c_u}\E[\|u\|^3\mid\mathcal F]
			\le c_r\mu^2.
			\end{align*}
			Moreover, \eqref{eq:q-generic} and the query-value moment bound imply
			\[
			\E[\|q_{i,h}(x;u,\zeta)\|^2] \le\frac{R_u^2}{c_u^2} \E[|\widetilde f_{i,h}(x+\mu u)|^2] \le\frac{R_u^2M_2}{c_u^2}=Q^2.
			\]
			
			For (ii), condition on $\mathcal F_{k,t}$, which fixes the memory table. For the expectation calculation, couple potential queries $\{q_{i,h,k}^t\}_{h=1}^{m_i}$ independently of the current batch selection; only sampled queries are evaluated. Uniform sampling in \eqref{eq:memory-estimator} gives
			\begin{align}
			\E[v_{i,k}^t\mid\mathcal F_{k,t}]
			&=\frac{1}{m_i}\sum_{h=1}^{m_i}
			\E[q_{i,h,k}^t\mid\mathcal F_{k,t}]\notag\\
			&=\mu_k\nabla f_i(\phi_{i,k}^t)+r_{i,k}^t.
			\label{eq:memory-cancellation}
			\end{align}
			since the sampled-memory mean cancels $\bar a_{i,k}^t$. Part (i) bounds each component remainder, and thus $\|r_{i,k}^t\|\le c_r\mu_k^2$.
			
			Each initialized entry has second moment at most $Q^2$. An entry is refreshed with probability $b_i/m_i$, independently of its current value. Therefore,
			\begin{align*}
			\E[\|a_{i,h,k}^{t+1}\|^2]
			\le{}&\left(1-\frac{b_i}{m_i}\right)
			\E[\|a_{i,h,k}^{t}\|^2]
			+\frac{b_i}{m_i}Q^2.
			\end{align*}
			Induction gives $\E[\|a_{i,h,k}^{t}\|^2]\le Q^2$ for each $h$ and $t$. This is an unconditional bound on the stored entries. Uniform sampling, Jensen's inequality, and total expectation then yield
			\begin{align*}
			\E\!\left[\left\|\frac{1}{b_i}
			\sum_{h\in\mathcal B_{i,k}^t}q_{i,h,k}^t\right\|^2\right]&\le Q^2,\\
			\E\!\left[\left\|\frac{1}{b_i}
			\sum_{h\in\mathcal B_{i,k}^t}a_{i,h,k}^t\right\|^2\right]&\le Q^2,
			\qquad \E[\|\bar a_{i,k}^t\|^2]\le Q^2.
			\end{align*}
			Applying $\|x-y+z\|^2\le3(\|x\|^2+\|y\|^2+\|z\|^2)$ to \eqref{eq:memory-estimator} gives $\E[\|v_{i,k}^t\|^2]\le9Q^2$, completing the proof.
		\end{IEEEproof}
		
		\section{Message Recursion and Transcript Privacy}
		\label{app:message-privacy}
		
		\begin{IEEEproof}[Proof of Lemma~\ref{lem:message}]
			For an oriented edge $e=(i,j)$, \eqref{eq:message}--\eqref{eq:z-update} give
			\begin{align}
			z_{ij,k+1}
			&=\frac{1}{2}
			\left(z_{ij,k}-z_{ji,k}+2\rho x_{j,k+1}\right),
			\notag\\
			z_{ji,k+1}
			&=\frac{1}{2}
			\left(z_{ji,k}-z_{ij,k}+2\rho x_{i,k+1}\right).
			\label{eq:two-z-updates}
			\end{align}
			Their difference gives \eqref{eq:omega-update}, and their sum gives
			\begin{align}
			z_{ij,k+1}+z_{ji,k+1}
			=\rho(x_{i,k+1}+x_{j,k+1}).
			\label{eq:z-sum-invariant}
			\end{align}
			The same sum identity holds at $k=0$, since $x_{i,0}=x_{j,0}=x_0$ and $z_{ij,0}=z_{ji,0}=\rho x_0$. Consequently,
			\begin{align}
			z_{ij,k}
			=B_{i,e}\omega_{e,k}
			+\frac{\rho}{2}(x_{i,k}+x_{j,k}).
			\label{eq:z-decomposition}
			\end{align}
			Substitution into \eqref{eq:penalty} yields
			\begin{align}
			p_k=\frac{\rho}{2}L_{\mathcal G}x_k-B\omega_k.
			\label{eq:p-current-omega}
			\end{align}
			Using $\omega_k=\omega_{k-1}-(\rho/2)B^\top x_k$ in \eqref{eq:p-current-omega} proves \eqref{eq:p-omega}. The definition $\lambda_k=-B\omega_{k-1}$ then gives \eqref{eq:p-lambda}--\eqref{eq:lambda-update}. At initialization, the shifted identity holds due to $\omega_{-1}=\omega_0=0$ and $B^\top(\one\otimes x_0)=0$. Finally, $\lambda_k\in\range(B)$, $\one^\top B=0$, and $\one^\top L_{\mathcal G}=0$ imply the zero-sum identities.
		\end{IEEEproof}
		
		\begin{IEEEproof}[Proof of Theorem~\ref{thm:privacy}]
			Fix client $i$ and condition on the transcript history $\mathcal H_k$. The vector $c_{i,k}:=x_{i,k}-\tau_i\eta_k\mu_k\beta p_{i,k}$ is then fixed. Let $\mathsf Q_{\mathcal D_i,k}$ be the conditional law of $\bar s_{i,k}$. Independence of the release noise gives
			\begin{equation}
			\mathsf P_{\mathcal D_i,k}
			=\int\mathcal N(c_{i,k}+s,\sigma_{\mathrm{dp},k}^2I_d)
			\,\mathsf Q_{\mathcal D_i,k}(\mathrm ds).
			\label{eq:conditional-release-mean}
			\end{equation}
			Both mixing laws under adjacent datasets are supported on $\{s:\|s\|\le R_k^{\mathrm{clip}}\}$, so
			\begin{align}
			\|s-s'\|\le2R_k^{\mathrm{clip}}.
			\label{eq:client-sensitivity}
			\end{align}
			For any R\'enyi order $\gamma>1$, use the common product measure $\mathsf Q_{\mathcal D_i,k}\otimes\mathsf Q_{\mathcal D_i',k}$ to couple the mixture centers. Data processing and the log-sum inequality give
			\begin{align}
			D_\gamma(\mathsf P_{\mathcal D_i,k}\|
			\mathsf P_{\mathcal D_i',k})
			\le\sup_{s,s'}D_\gamma\!\Bigl(&
			\mathcal N(c_{i,k}+s,\sigma_{\mathrm{dp},k}^2I_d)
			\notag\\
			&\Big\|\mathcal N(c_{i,k}+s',\sigma_{\mathrm{dp},k}^2I_d)
			\Bigr).
			\label{eq:mixture-renyi}
			\end{align}
			The equal-covariance Gaussian formula and \eqref{eq:client-sensitivity} imply
			\begin{equation}
			D_\gamma(\mathsf P_{\mathcal D_i,k}\| \mathsf P_{\mathcal D_i',k}) \le\frac{\gamma(2R_k^{\mathrm{clip}})^2} {2\sigma_{\mathrm{dp},k}^2} =\frac{2\gamma(R_k^{\mathrm{clip}})^2} {\sigma_{\mathrm{dp},k}^2}.
			\label{eq:mixture-zcdp}
			\end{equation}
			Thus, client $i$'s conditional release is $2(R_k^{\mathrm{clip}})^2/\sigma_{\mathrm{dp},k}^2$-zCDP.
			
			All outgoing messages are deterministic post-processing of the released state and history. Under client-$i$ adjacency, the other clients' datasets are fixed; their responses use the transcript and fresh independent randomness, adding no separate privacy charge for client $i$. Adaptive composition yields the cumulative zCDP budget, whose conversion to approximate DP gives the stated privacy guarantee. The adjacency relation replaces the entire local dataset, establishing the stated client-level guarantee.
		\end{IEEEproof}
		
		\section{Schur Stability and Disagreement Bounds}
		\label{app:stability}
		
		\begin{IEEEproof}[Proof of Lemma~\ref{lem:stability}]
			Summing \eqref{eq:local-update} and applying \eqref{eq:private-release} gives
			\begin{align}
			x_{k+1}
			=x_k-a p_k+w_k,
			\label{eq:stacked-x-update}\\
			w_k
			=-\eta\sum_{t=0}^{\tau-1}v_k^t
			+e_k^{\mathrm{cl}}+\nu_k.
			\label{eq:w-definition}
			\end{align}
			Equations \eqref{eq:p-lambda}--\eqref{eq:lambda-update} then yield the modal recursion
			\begin{align}
			\begin{bmatrix}
			\widehat x_{\lambda,k+1}\\
			\widehat\lambda_{\lambda,k+1}
			\end{bmatrix}
			=M_\lambda
			\begin{bmatrix}
			\widehat x_{\lambda,k}\\
			\widehat\lambda_{\lambda,k}
			\end{bmatrix}
			+
			\begin{bmatrix}
			\widehat w_{\lambda,k}\\
			0
			\end{bmatrix}.
			\label{eq:modal-dynamics}
			\end{align}
			The trace and determinant are
			\begin{equation}
			\operatorname{tr}(M_\lambda)=2-a\rho\lambda,
			\qquad
			\det(M_\lambda)=1-\frac{a\rho\lambda}{2}.
			\label{eq:trace-det}
			\end{equation}
			The second-order Jury criterion gives
			\begin{equation}
			\begin{aligned}[b]
			1-\det(M_\lambda)&=\frac{a\rho\lambda}{2},\\
			1-\operatorname{tr}(M_\lambda)+\det(M_\lambda)
			&=\frac{a\rho\lambda}{2},\\
			1+\operatorname{tr}(M_\lambda)+\det(M_\lambda)
			&=4-\frac{3a\rho\lambda}{2}.
			\end{aligned}
			\label{eq:jury-expressions}
			\end{equation}
			These three expressions are positive exactly when $0<a\rho\lambda<8/3$. Requiring this for each nonzero Laplacian eigenvalue gives \eqref{eq:stability}.
			
			For a Schur-stable $M_\lambda$, the convergent series
			\begin{align}
			P_\lambda
			=\sum_{r=0}^{\infty}
			(M_\lambda^r)^\top M_\lambda^r
			\label{eq:P-series}
			\end{align}
			is the unique positive-definite solution of \eqref{eq:lyapunov}. In the ordering of $\xi_k$, set
			\[
			M:=\mathsf S^\top
			\diag_{\ell=2}^N(M_{\lambda_\ell}\otimes I_d)\mathsf S.
			\]
			Then $M^\top\mathcal P M-\mathcal P=-I_{2(N-1)d}$. Writing $\|\xi\|_{\mathcal P}^2:=\xi^\top\mathcal P\xi$, we obtain
			\begin{align}
			\|M\xi\|_{\mathcal P}^2
			=\|\xi\|_{\mathcal P}^2-\|\xi\|^2
			\le\left(1-\frac{1}{\overline p}\right)
			\|\xi\|_{\mathcal P}^2.
			\label{eq:unforced-contraction}
			\end{align}
			Since $\overline p>1$, Young's inequality with parameter $[2(\overline p-1)]^{-1}$ yields
			\begin{align}
			\|M\xi+\operatorname{col}(\widehat w,0)\|_{\mathcal P}^2
			\le{}&\left(1-\frac{1}{2\overline p}\right)\|\xi\|_{\mathcal P}^2
			\notag\\
			&+\overline p(2\overline p-1)\|\widehat w\|^2.
			\label{eq:forced-contraction}
			\end{align}
			Here $\widehat w=(U_\perp^\top\!\otimes I_d)w$, so $\|\widehat w\|=\|\Pi w\|$.
			
			The release noise is independent and zero-mean, eliminating its cross terms with the local update and clipping error. Since $\Pi$ is nonexpansive, Cauchy--Schwarz and Lemma~\ref{lem:oracle_memory}(ii) give
			\begin{align}
			&\frac{1}{N}\E[\|\Pi w_k\|^2]
			\le\frac{2\eta^2\tau}{N}
			\sum_{t=0}^{\tau-1}\E[\|v_k^t\|^2]+2E_k^{\mathrm{cl}}
			\notag\\
			&\quad+\left(1-\frac{1}{N}\right)d\sigma_{\mathrm{dp},k}^2
			\notag\\
			&\quad\le18\eta^2\tau^2Q^2+2E_k^{\mathrm{cl}}
			+\left(1-\frac{1}{N}\right)d\sigma_{\mathrm{dp},k}^2.
			\label{eq:w-bound-start}
			\end{align}
			Taking expectations in \eqref{eq:forced-contraction} and dividing by $N$ proves \eqref{eq:V-recursion}.
			
			The eigenvalue bounds on $\mathcal P$ imply \eqref{eq:C-by-V}. Using $p_k=\rho L_{\mathcal G}x_k+\lambda_k$ further gives
			\[
			\frac{1}{N}\E[\|p_k\|^2] \le2\rho^2\lambda_N^2\mathcal C_k +\frac{2}{N}\E[\|\lambda_k\|^2] \le\frac{2(\rho^2\lambda_N^2+1)}{\underline p}\mathcal V_k,
			\]
			which proves \eqref{eq:p-by-V}. The common initialization has zero primal disagreement and $\lambda_0=0$, hence $\mathcal V_0=0$.
		\end{IEEEproof}
		
		\section{Descent and Convergence Bounds}
		\label{app:descent-main}
		
		\begin{IEEEproof}[Proof of Lemma~\ref{lem:local_disagreement}]
			Iterating \eqref{eq:local-update} gives
			\begin{align}
			\phi_k^t
			=x_k-\eta\sum_{s=0}^{t-1}v_k^s
			-t\alpha\beta p_k.
			\label{eq:phi-expansion}
			\end{align}
			Since $\Pi p_k=p_k$, applying $\Pi$ and $\|a+b+c\|^2\le3(\|a\|^2+\|b\|^2+\|c\|^2)$ yields, by Lemma~\ref{lem:oracle_memory}(ii),
			\begin{align}
			\mathcal C_{k,t}^{\phi}
			\le
			3\mathcal C_k
			+27\eta^2t^2Q^2
			+3t^2\alpha^2\beta^2
			\frac{1}{N}\E\!\left[\|p_k\|^2\right].
			\label{eq:Cphi-intermediate}
			\end{align}
			Equations \eqref{eq:C-by-V}--\eqref{eq:p-by-V} and $t\le\tau-1$ give \eqref{eq:local-disagreement}.
		\end{IEEEproof}
		
		\begin{IEEEproof}[Proof of Lemma~\ref{lem:descent}]
			The zero-sum identity $\one^\top p_k=0$ gives
			\begin{align}
			\bar\phi_k^{t+1}
			=\bar\phi_k^t-\eta\bar v_k^t,
			\qquad
			\bar v_k^t=\frac{1}{N}\sum_{i=1}^{N}v_{i,k}^t.
			\label{eq:average-local-update}
			\end{align}
			Set $G_k^t:=\nabla F(\bar\phi_k^t)$ and $g_k^t:=N^{-1}\sum_i\nabla f_i(\phi_{i,k}^t)$. Smoothness and Jensen's inequality imply
			\begin{align}
			\E\!\left[\|g_k^t-G_k^t\|^2\right]
			\le L_f^2\mathcal C_{k,t}^{\phi}.
			\label{eq:gradient-mismatch}
			\end{align}
			Lemma~\ref{lem:oracle_memory}(ii) also gives
			\begin{align}
			\E\!\left[\bar v_k^t\mid\mathcal F_{k,t}\right]
			=\mu g_k^t+\bar r_k^t,
			\qquad
			\|\bar r_k^t\|\le c_r\mu^2.
			\label{eq:average-v-mean}
			\end{align}
			where $\bar r_k^t:=N^{-1}\sum_i r_{i,k}^t$. By smoothness, conditional expectation, and $\E[\|\bar v_k^t\|^2]\le9Q^2$,
			\begin{align}
			\E\!\left[F(\bar\phi_k^{t+1})\right]
			\le{}&
			\E\!\left[F(\bar\phi_k^t)\right]
			-\alpha\E\!\left[\langle G_k^t,g_k^t\rangle\right]
			\notag\\
			&-\eta\E\!\left[\langle G_k^t,\bar r_k^t\rangle\right]
			+\frac{9L_f}{2}\eta^2Q^2.
			\label{eq:descent-start}
			\end{align}
			The inequalities
			\begin{align}
			\langle G_k^t,g_k^t\rangle
			&\ge\frac{1}{2}\|G_k^t\|^2
			-\frac{1}{2}\|g_k^t-G_k^t\|^2,
			\notag\\
			\eta|\langle G_k^t,\bar r_k^t\rangle|
			&\le\frac{\alpha}{4}\|G_k^t\|^2
			+\alpha c_r^2\mu^2
			\label{eq:descent-young}
			\end{align}
			therefore yield
			\begin{align}
			\E\!\left[F(\bar\phi_k^{t+1})\right]
			\le{}&
			\E\!\left[F(\bar\phi_k^t)\right]
			-\frac{\alpha}{4}\E\!\left[\|G_k^t\|^2\right]
			\notag\\
			&+\frac{\alpha L_f^2}{2}
			\mathcal C_{k,t}^{\phi}
			+\alpha c_r^2\mu^2
			+\frac{9L_f}{2}\eta^2Q^2.
			\label{eq:one-step-descent}
			\end{align}
			Sum over $t$ and use Lemma~\ref{lem:local_disagreement}. Since $\sum_{t=0}^{\tau-1}t^2=\tau(\tau-1)(2\tau-1)/6$, the definition of $A_{\mathrm{loc}}$ gives
			\begin{align}
			\E[F(\bar\phi_k^\tau)]
			\le{}&\E[F(\bar x_k)]
			-\frac{\alpha}{4}\sum_{t=0}^{\tau-1}\E[\|G_k^t\|^2]
			\notag\\
			&+\frac{\alpha L_f^2\tau\Gamma}{2}\mathcal V_k
			+A_{\mathrm{loc}}
			-\frac{9}{8}\alpha L_f^2\eta^2Q^2.
			\label{eq:pre-release-descent}
			\end{align}
			
			The average release satisfies
			\begin{align}
			\bar x_{k+1}
			=\bar\phi_k^\tau
			+\bar e_k^{\mathrm{cl}}+\bar\nu_k.
			\label{eq:average-release}
			\end{align}
			Conditionally on the pre-noise variables, $\bar\nu_k$ is zero-mean and
			\begin{align}
			\E\!\left[\|\bar\nu_k\|^2\right]
			=\frac{d\sigma_{\mathrm{dp},k}^2}{N}.
			\label{eq:average-noise}
			\end{align}
			A second application of smoothness yields
			\begin{align}
			\E\!\left[F(\bar x_{k+1})\right]
			\le{}&
			\E\!\left[F(\bar\phi_k^\tau)\right]
			+\E\!\left[\langle\nabla F(\bar\phi_k^\tau),
			\bar e_k^{\mathrm{cl}}\rangle\right]
			\notag\\
			&+\frac{L_f}{2}\E\!\left[\|\bar e_k^{\mathrm{cl}}\|^2\right]
			+\frac{L_f}{2N}d\sigma_{\mathrm{dp},k}^2.
			\label{eq:release-smoothness}
			\end{align}
			Young's inequality gives
			\begin{align}
			\langle\nabla F(\bar\phi_k^\tau),\bar e_k^{\mathrm{cl}}\rangle
			\le\frac{\alpha}{16}
			\|\nabla F(\bar\phi_k^\tau)\|^2
			+\frac{4}{\alpha}\|\bar e_k^{\mathrm{cl}}\|^2.
			\label{eq:clip-young}
			\end{align}
			Since $\bar\phi_k^\tau=\bar\phi_k^{\tau-1}-\eta\bar v_k^{\tau-1}$, smoothness also gives
			\begin{align}
			\E\!\left[\|\nabla F(\bar\phi_k^\tau)\|^2\right]
			\le2\E\!\left[\|G_k^{\tau-1}\|^2\right]
			+18L_f^2\eta^2Q^2.
			\label{eq:terminal-gradient}
			\end{align}
			Finally, Jensen's inequality yields $\E[\|\bar e_k^{\mathrm{cl}}\|^2]\le E_k^{\mathrm{cl}}$. Substituting \eqref{eq:clip-young}--\eqref{eq:terminal-gradient} into \eqref{eq:release-smoothness} adds at most $(\alpha/8)\E[\|G_k^{\tau-1}\|^2]$ to the gradient terms. The accompanying $(9/8)\alpha L_f^2\eta^2Q^2$ cancels the last term in \eqref{eq:pre-release-descent}. Bounding each remaining gradient coefficient by $-\alpha/8$ proves \eqref{eq:round-descent}.
		\end{IEEEproof}
		
		\begin{IEEEproof}[Proof of Theorem~\ref{thm:main}]
			Define
			\[
			\mathcal R_K:=
			18K\eta^2\tau^2Q^2+2\mathcal E_{\mathrm{cl},K}
			+\left(1-\frac{1}{N}\right)\mathcal E_{\mathrm{dp},K}.
			\]
			Summing \eqref{eq:V-recursion} and using $\mathcal V_K\ge0$ gives
			\begin{align}
			\chi\sum_{k=0}^{K-1}\mathcal V_k
			\le\mathcal V_0+c_w\mathcal R_K.
			\label{eq:V-sum}
			\end{align}
			Equation \eqref{eq:C-by-V} then implies
			\begin{align}
			\frac{1}{K}\sum_{k=0}^{K-1}\mathcal C_k
			\le
			\frac{\mathcal V_0+c_w\mathcal R_K}
			{K\chi\underline p}.
			\label{eq:consensus-sum}
			\end{align}
			Summing \eqref{eq:round-descent} and using $F(\bar x_K)\ge F_\star$ yields
			\begingroup
			\renewcommand{\mintagsep}{3pt}
			\begin{align}
			&\frac{\alpha}{8}\sum_{k=0}^{K-1}\sum_{t=0}^{\tau-1}
			\E[\|\nabla F(\bar\phi_k^t)\|^2]
			\le F(\bar x_0)-F_\star+KA_{\mathrm{loc}}
			\notag\\
			&\quad+\frac{\alpha L_f^2\tau\Gamma}{2}
			\sum_{k=0}^{K-1}\mathcal V_k
			+\left(\frac{4}{\alpha}+\frac{L_f}{2}\right)\mathcal E_{\mathrm{cl},K}
			+\frac{L_f}{2N}\mathcal E_{\mathrm{dp},K}.
			\label{eq:gradient-sum}
			\end{align}
			\endgroup
			Substitute \eqref{eq:V-sum}, divide by $K\tau\alpha/8$, and add \eqref{eq:consensus-sum}. This gives
			\begin{align}
			\mathcal S_K\le{}&
			\frac{8[F(\bar x_0)-F_\star]}{K\tau\alpha}
			+\frac{8A_{\mathrm{loc}}}{\tau\alpha}\notag\\
			&+\left(4L_f^2\Gamma+\frac{1}{\underline p}\right)
			\frac{\mathcal V_0+c_w\mathcal R_K}{K\chi}\notag\\
			&+\left(\frac{32}{\tau\alpha^2}+\frac{4L_f}{\tau\alpha}\right)
			\overline{\mathcal E}_{\mathrm{cl},K}
			+\frac{4L_f}{N\tau\alpha}\overline{\mathcal E}_{\mathrm{dp},K}.
			\label{eq:explicit-main-bound}
			\end{align}
			The common initialization gives $\mathcal V_0=0$. For fixed $a=\tau\beta\alpha$, $\rho$, and graph,
			\[
			(\tau-1)^2\alpha^2\beta^2
			=\left(\frac{\tau-1}{\tau}\right)^2a^2\le a^2.
			\]
			Thus, $\Gamma$ is bounded uniformly over $\tau$, and the Lyapunov constants are fixed. Directly from their definitions,
			\begin{align*}
			\frac{A_{\mathrm{loc}}}{\tau\alpha}
			&=\mathcal O\!\left(\frac{\eta}{\mu}+\mu^2+\eta^2\tau^2\right),\\
			\frac{\mathcal R_K}{K}
			&=\mathcal O\!\left(\eta^2\tau^2
			+\overline{\mathcal E}_{\mathrm{cl},K}
			+\overline{\mathcal E}_{\mathrm{dp},K}\right).
			\end{align*}
			For $\alpha>0$ and $\tau\ge1$, $1/(\tau\alpha)\le1+1/(\tau\alpha^2)$. Substitution into \eqref{eq:explicit-main-bound} establishes \eqref{eq:main-rate}.
		\end{IEEEproof}
		
		\begin{IEEEproof}[Proof of Corollary~\ref{cor:nonprivate}]
			The schedule \eqref{eq:nonprivate-schedule} fixes $a=\tau\beta_{T,\tau}\eta_T\mu_T=a_0$. Since $T=K\tau$, each of $1/(T\eta_T\mu_T)$, $\eta_T/\mu_T$, and $\mu_T^2$ is $\mathcal O(T^{-1/3})$, while $\eta_T^2\tau^2=\mathcal O(\tau^2/T)$. The release terms vanish, so \eqref{eq:main-rate} gives \eqref{eq:nonprivate-rate}. The stated local-update range and communication count follow from $\tau^2/T=\mathcal O(T^{-1/3})$ and $K=T/\tau$.
		\end{IEEEproof}
	\end{spadecompact}

\bibliographystyle{IEEEtran}
\bibliography{spade_dfl_v1_refs}

\begin{thebibliography}{10}
\providecommand{\url}[1]{#1}
\csname url@samestyle\endcsname
\providecommand{\newblock}{\relax}
\providecommand{\bibinfo}[2]{#2}
\providecommand{\BIBentrySTDinterwordspacing}{\spaceskip=0pt\relax}
\providecommand{\BIBentryALTinterwordstretchfactor}{4}
\providecommand{\BIBentryALTinterwordspacing}{\spaceskip=\fontdimen2\font plus
\BIBentryALTinterwordstretchfactor\fontdimen3\font minus
  \fontdimen4\font\relax}
\providecommand{\BIBforeignlanguage}[2]{{%
\expandafter\ifx\csname l@#1\endcsname\relax
\typeout{** WARNING: IEEEtran.bst: No hyphenation pattern has been}%
\typeout{** loaded for the language `#1'. Using the pattern for}%
\typeout{** the default language instead.}%
\else
\language=\csname l@#1\endcsname
\fi
#2}}
\providecommand{\BIBdecl}{\relax}
\BIBdecl

\bibitem{zehtabi2025decentralized}
S.~Zehtabi, D.-J. Han, R.~Parasnis, S.~Hosseinalipour, and C.~G. Brinton,
  ``Decentralized sporadic federated learning: A unified algorithmic framework
  with convergence guarantees,'' in \emph{The Thirteenth International
  Conference on Learning Representations}, 2025.

\bibitem{wu2025localupdates}
T.~Wu, Z.~Li, and Y.~Sun, ``The effectiveness of local updates for
  decentralized learning under data heterogeneity,'' \emph{IEEE Transactions on
  Signal Processing}, vol.~73, pp. 751--765, 2025.

\bibitem{alghunaim2024led}
S.~A. Alghunaim, ``Local exact-diffusion for decentralized optimization and
  learning,'' \emph{IEEE Transactions on Automatic Control}, vol.~69, no.~11,
  pp. 7371--7386, 2024.

\bibitem{ren2026communication}
X.~Ren, N.~Bastianello, K.~H. Johansson, and T.~Parisini,
  ``Communication-efficient stochastic distributed learning,'' \emph{IEEE
  Transactions on Automatic Control}, vol.~71, no.~9, pp. 5741--5756, 2026.

\bibitem{ghadimi2013stochastic}
S.~Ghadimi and G.~Lan, ``Stochastic first- and zeroth-order methods for
  nonconvex stochastic programming,'' \emph{SIAM Journal on Optimization},
  vol.~23, no.~4, pp. 2341--2368, 2013.

\bibitem{ye2025hessian}
H.~Ye, Z.~Huang, C.~Fang, C.~J. Li, and T.~Zhang, ``Hessian-aware zeroth-order
  optimization,'' \emph{IEEE Transactions on Pattern Analysis and Machine
  Intelligence}, vol.~47, no.~6, pp. 4869--4877, 2025.

\bibitem{huang2024zoadmm}
F.~Huang, S.~Gao, J.~Pei, and H.~Huang, ``Nonconvex zeroth-order stochastic
  {ADMM} methods with lower function query complexity,'' \emph{IEEE
  Transactions on Pattern Analysis and Machine Intelligence}, pp. 1--13, 2024,
  early Access.

\bibitem{mhanna2023single}
E.~Mhanna and M.~Assaad, ``Single point-based distributed zeroth-order
  optimization with a non-convex stochastic objective function,'' in
  \emph{Proceedings of the 40th International Conference on Machine Learning},
  ser. Proceedings of Machine Learning Research, vol. 202.\hskip 1em plus 0.5em
  minus 0.4em\relax PMLR, 2023, pp. 24\,701--24\,719.

\bibitem{mhanna2024zero}
------, ``Zero-order one-point gradient estimate in consensus-based distributed
  stochastic optimization,'' \emph{Transactions on Machine Learning Research},
  Nov. 2024.

\bibitem{song2022compressedgt}
Z.~Song, L.~Shi, S.~Pu, and M.~Yan, ``Compressed gradient tracking for
  decentralized optimization over general directed networks,'' \emph{IEEE
  Transactions on Signal Processing}, vol.~70, pp. 1775--1787, 2022.

\bibitem{nassif2025def}
R.~Nassif, S.~Vlaski, M.~Carpentiero, V.~Matta, and A.~H. Sayed, ``Differential
  error feedback for communication-efficient decentralized learning,''
  \emph{IEEE Transactions on Signal Processing}, vol.~73, pp. 1905--1921, 2025.

\bibitem{he2023unbiased}
Y.~He, X.~Huang, and K.~Yuan, ``Unbiased compression saves communication in
  distributed optimization: When and how much?'' in \emph{Advances in Neural
  Information Processing Systems}, vol.~36, 2023, pp. 47\,991--48\,020.

\bibitem{guo2026gradient}
P.~Guo, R.~Wang, S.~Zeng, J.~Zhu, H.~Jiang, Y.~Wang, Y.~Zhou, F.~Wang,
  H.~Xiong, and L.~Qu, ``Exploring the vulnerabilities of federated learning: A
  deep dive into gradient inversion attacks,'' \emph{IEEE Transactions on
  Pattern Analysis and Machine Intelligence}, vol.~48, no.~4, pp. 4810--4826,
  2026.

\bibitem{rizk2023privacy}
E.~Rizk, S.~Vlaski, and A.~H. Sayed, ``Enforcing privacy in distributed
  learning with performance guarantees,'' \emph{IEEE Transactions on Signal
  Processing}, vol.~71, pp. 3385--3398, 2023.

\bibitem{allouah2024decor}
Y.~Allouah, A.~Koloskova, A.~{El Firdoussi}, M.~Jaggi, and R.~Guerraoui, ``The
  privacy power of correlated noise in decentralized learning,'' in
  \emph{Proceedings of the 41st International Conference on Machine Learning},
  ser. Proceedings of Machine Learning Research, vol. 235.\hskip 1em plus 0.5em
  minus 0.4em\relax PMLR, 2024, pp. 1115--1143.

\bibitem{cyffers2023privateadmm}
E.~Cyffers, A.~Bellet, and D.~Basu, ``From noisy fixed-point iterations to
  private {ADMM} for centralized and federated learning,'' in \emph{Proceedings
  of the 40th International Conference on Machine Learning}, ser. Proceedings
  of Machine Learning Research, vol. 202.\hskip 1em plus 0.5em minus
  0.4em\relax PMLR, 2023, pp. 6683--6711.

\bibitem{bun2016concentrated}
M.~Bun and T.~Steinke, ``Concentrated differential privacy: Simplifications,
  extensions, and lower bounds,'' in \emph{Theory of Cryptography
  Conference}.\hskip 1em plus 0.5em minus 0.4em\relax Springer, 2016, pp.
  635--658.

\bibitem{mironov2017renyi}
I.~Mironov, ``R{\'e}nyi differential privacy,'' in \emph{2017 IEEE 30th
  Computer Security Foundations Symposium}, 2017, pp. 263--275.

\bibitem{boyd2011distributed}
S.~Boyd, N.~Parikh, E.~Chu, B.~Peleato, and J.~Eckstein, ``Distributed
  optimization and statistical learning via the alternating direction method of
  multipliers,'' \emph{Foundations and Trends in Machine Learning}, vol.~3,
  no.~1, pp. 1--122, 2011.

\bibitem{shi2014linearized}
W.~Shi, Q.~Ling, K.~Yuan, G.~Wu, and W.~Yin, ``On the linear convergence of the
  {ADMM} in decentralized consensus optimization,'' \emph{IEEE Transactions on
  Signal Processing}, vol.~62, no.~7, pp. 1750--1761, 2014.

\bibitem{ling2015dlm}
Q.~Ling, W.~Shi, G.~Wu, and A.~Ribeiro, ``{DLM}: Decentralized linearized
  alternating direction method of multipliers,'' \emph{IEEE Transactions on
  Signal Processing}, vol.~63, no.~15, pp. 4051--4064, 2015.

\bibitem{li2023curvature}
Y.~Li, P.~G. Voulgaris, D.~M. Stipanovi{\'c}, and N.~M. Freris, ``Communication
  efficient curvature aided primal-dual algorithms for decentralized
  optimization,'' \emph{IEEE Transactions on Automatic Control}, vol.~68,
  no.~11, pp. 6573--6588, 2023.

\bibitem{gautam2024mezosvrg}
T.~Gautam, Y.~Park, H.~Zhou, P.~Raman, and W.~Ha, ``Variance-reduced
  zeroth-order methods for fine-tuning language models,'' in \emph{Proceedings
  of the 41st International Conference on Machine Learning}, ser. Proceedings
  of Machine Learning Research, vol. 235.\hskip 1em plus 0.5em minus
  0.4em\relax PMLR, 2024, pp. 15\,180--15\,208.

\bibitem{koloskova2023clipping}
A.~Koloskova, H.~Hendrikx, and S.~U. Stich, ``Revisiting gradient clipping:
  Stochastic bias and tight convergence guarantees,'' in \emph{Proceedings of
  the 40th International Conference on Machine Learning}, ser. Proceedings of
  Machine Learning Research, vol. 202.\hskip 1em plus 0.5em minus 0.4em\relax
  PMLR, 2023, pp. 17\,343--17\,363.

\bibitem{mishchenko2022proxskip}
K.~Mishchenko, G.~Malinovsky, S.~Stich, and P.~Richt{\'a}rik, ``{ProxSkip}:
  Yes! {L}ocal gradient steps provably lead to communication acceleration!
  {F}inally!'' in \emph{Proceedings of the 39th International Conference on
  Machine Learning}, ser. Proceedings of Machine Learning Research, vol.
  162.\hskip 1em plus 0.5em minus 0.4em\relax PMLR, 2022, pp. 15\,750--15\,769.

\bibitem{ren2026ltadmmdp}
X.~Ren, Y.~Ma, N.~Bastianello, K.~H. Johansson, T.~Parisini, and A.~A.
  Malikopoulos, ``Communication-efficient distributed learning with
  differential privacy,'' \emph{arXiv preprint arXiv:2604.02558}, 2026.

\bibitem{ding2023dsgdceca}
L.~Ding, K.~Jin, B.~Ying, K.~Yuan, and W.~Yin, ``{DSGD-CECA}: Decentralized
  {SGD} with communication-optimal exact consensus algorithm,'' in
  \emph{Proceedings of the 40th International Conference on Machine Learning},
  ser. Proceedings of Machine Learning Research, vol. 202.\hskip 1em plus 0.5em
  minus 0.4em\relax PMLR, 2023, pp. 8067--8089.

\bibitem{you2024btpp}
R.~You and S.~Pu, ``B-ary tree push-pull method is provably efficient for
  distributed learning on heterogeneous data,'' in \emph{Advances in Neural
  Information Processing Systems}, vol.~37, 2024, pp. 97\,523--97\,561.

\bibitem{yuan2023heterogeneity}
K.~Yuan, S.~A. Alghunaim, and X.~Huang, ``Removing data heterogeneity influence
  enhances network topology dependence of decentralized {SGD},'' \emph{Journal
  of Machine Learning Research}, vol.~24, no. 280, pp. 1--53, 2023.

\bibitem{zhai2026aggregation}
Z.~Zhai, X.~Yuan, X.~Wang, and G.~Y. Li, ``Decentralized federated learning
  with distributed aggregation weight optimization,'' \emph{IEEE Transactions
  on Pattern Analysis and Machine Intelligence}, vol.~48, no.~3, pp.
  3899--3910, 2026.

\bibitem{sun2026generalization}
Y.~Sun, L.~Shen, and D.~Tao, ``Toward understanding generalization and
  stability gaps between centralized and decentralized federated learning,''
  \emph{IEEE Transactions on Pattern Analysis and Machine Intelligence},
  vol.~48, no.~4, pp. 4744--4755, 2026.

\bibitem{zhao2022beer}
H.~Zhao, B.~Li, Z.~Li, P.~Richt{\'a}rik, and Y.~Chi, ``{BEER}: Fast {$O(1/T)$}
  rate for decentralized nonconvex optimization with communication
  compression,'' in \emph{Advances in Neural Information Processing Systems},
  vol.~35, 2022, pp. 31\,653--31\,667.

\bibitem{islamov2025motef}
R.~Islamov, Y.~Gao, and S.~U. Stich, ``Towards faster decentralized stochastic
  optimization with communication compression,'' in \emph{Proceedings of the
  13th International Conference on Learning Representations}, 2025.

\bibitem{li2024onlinecompressed}
J.~Li, C.~Li, J.~Fan, and T.~Huang, ``Online distributed stochastic gradient
  algorithm for nonconvex optimization with compressed communication,''
  \emph{IEEE Transactions on Automatic Control}, vol.~69, no.~2, pp. 936--951,
  2024.

\bibitem{hua2026distributed}
Y.~Hua, S.~Liu, Y.~Hong, and W.~Ren, ``Distributed stochastic zeroth-order
  optimization with compressed communication,'' \emph{IEEE Transactions on
  Automatic Control}, vol.~71, no.~2, pp. 1294--1301, 2026.

\bibitem{xu2024quantized}
L.~Xu, X.~Yi, C.~Deng, Y.~Shi, T.~Chai, and T.~Yang, ``Quantized zeroth-order
  gradient tracking algorithm for distributed nonconvex optimization under
  {Polyak--{\L}ojasiewicz} condition,'' \emph{IEEE Transactions on
  Cybernetics}, vol.~54, no.~10, pp. 5746--5758, 2024.

\bibitem{fang2022communication}
W.~Fang, Z.~Yu, Y.~Jiang, Y.~Shi, C.~N. Jones, and Y.~Zhou,
  ``Communication-efficient stochastic zeroth-order optimization for federated
  learning,'' \emph{IEEE Transactions on Signal Processing}, vol.~70, pp.
  5058--5073, 2022.

\bibitem{li2025decomfl}
Z.~Li, B.~Ying, Z.~Liu, C.~Dong, and H.~Yang, ``Achieving dimension-free
  communication in federated learning via zeroth-order optimization,'' in
  \emph{Proceedings of the 13th International Conference on Learning
  Representations}, 2025.

\bibitem{malladi2023mezo}
S.~Malladi, T.~Gao, E.~Nichani, A.~Damian, J.~D. Lee, D.~Chen, and S.~Arora,
  ``Fine-tuning language models with just forward passes,'' in \emph{Advances
  in Neural Information Processing Systems}, vol.~36, 2023, pp.
  53\,038--53\,075.

\bibitem{li2024adqsp}
Q.~Li, J.~S. Gundersen, M.~Lopuha{\"a}-Zwakenberg, and R.~Heusdens, ``Adaptive
  differentially quantized subspace perturbation ({ADQSP}): A unified framework
  for privacy-preserving distributed average consensus,'' \emph{IEEE
  Transactions on Information Forensics and Security}, vol.~19, pp. 1780--1793,
  2024.

\bibitem{wang2026ping}
L.~Wang, S.~Yang, Y.~Wan, W.~Xu, and M.-L. Zhang, ``Privacy preserving
  decentralized learning with positive-incentive noise,'' \emph{IEEE
  Transactions on Pattern Analysis and Machine Intelligence}, vol.~48, no.~7,
  pp. 8520--8534, 2026.

\bibitem{cyffers2022muffliato}
E.~Cyffers, M.~Even, A.~Bellet, and L.~Massouli\'{e}, ``Muffliato: Peer-to-peer
  privacy amplification for decentralized optimization and averaging,'' in
  \emph{Advances in Neural Information Processing Systems}, vol.~35, 2022, pp.
  15\,889--15\,902.

\bibitem{zhang2024dpzero}
L.~Zhang, B.~Li, K.~K. Thekumparampil, S.~Oh, and N.~He, ``{DPZero}: Private
  fine-tuning of language models without backpropagation,'' in
  \emph{Proceedings of the 41st International Conference on Machine Learning},
  ser. Proceedings of Machine Learning Research, vol. 235.\hskip 1em plus 0.5em
  minus 0.4em\relax PMLR, 2024, pp. 59\,210--59\,246.

\bibitem{gong2025pazo}
X.~Gong and T.~Li, ``Private zeroth-order optimization with public data,'' in
  \emph{Advances in Neural Information Processing Systems}, vol.~38, 2025, pp.
  58\,619--58\,665.

\bibitem{shi2025dpfedsam}
Y.~Shi, K.~Wei, L.~Shen, Y.~Liu, X.~Wang, B.~Yuan, and D.~Tao, ``Toward the
  flatter landscape and better generalization in federated learning under
  client-level differential privacy,'' \emph{IEEE Transactions on Pattern
  Analysis and Machine Intelligence}, vol.~47, no.~12, pp. 11\,632--11\,643,
  2025.

\bibitem{mcmahan2017communication}
B.~McMahan, E.~Moore, D.~Ramage, S.~Hampson, and B.~Ag{\"u}era~y Arcas,
  ``Communication-efficient learning of deep networks from decentralized
  data,'' in \emph{Proceedings of the 20th International Conference on
  Artificial Intelligence and Statistics}, ser. Proceedings of Machine Learning
  Research, vol.~54.\hskip 1em plus 0.5em minus 0.4em\relax PMLR, 2017, pp.
  1273--1282.

\bibitem{kairouz2021advances}
P.~Kairouz, H.~B. McMahan, B.~Avent \emph{et~al.}, ``Advances and open problems
  in federated learning,'' \emph{Foundations and Trends in Machine Learning},
  vol.~14, no. 1--2, pp. 1--210, 2021.

\end{thebibliography}

\end{document}